\documentclass{article}

\usepackage[utf8]{inputenc}
\usepackage[utf8]{inputenc} 
\usepackage[T1]{fontenc}    
\usepackage{hyperref}   
\hypersetup{
	colorlinks=true, 
	linkcolor=red,  
	anchorcolor=blue, 
	citecolor=blue, 
	filecolor=magenta, 
	menucolor=red, 
	urlcolor=cyan, 
}
\usepackage{url}            
\usepackage{booktabs}       
\usepackage{amsfonts}       
\usepackage{nicefrac}       
\usepackage{microtype}      
\usepackage[dvipsnames]{xcolor}         

\usepackage{amsmath, amsthm, amsfonts, amssymb, mathrsfs, float}
\usepackage{algorithm}
\usepackage{algorithmic}
\usepackage{geometry}
\usepackage{comment}
\usepackage{bm}
\usepackage{verbatim}
\usepackage{mathtools}
\usepackage{bbm}
\usepackage{dsfont}
\usepackage{booktabs}
\usepackage{relsize}
\usepackage{caption}
\usepackage{subcaption}
\usepackage[english]{babel}
\usepackage{natbib}
\usepackage{tabulary}
\usepackage{colortbl}
\usepackage{enumitem}
\usepackage{tablefootnote}
\usepackage[capitalize,noabbrev]{cleveref}

\newtheorem{assumption}{Assumption}
\newtheorem{lemma}{Lemma}
\newtheorem{proposition}{Proposition}
\newtheorem{theorem}{Theorem}
\newtheorem{corollary}{Corollary}

\DeclareMathOperator{\avg}{avg}
\def\bx{\bm{x}}

\def\by{\bm{y}}

\def\bw{\bm{w}}

\def\bg{\bm{g}}
\def\bG{\bm{G}}

\def\bxi{\bm{\xi}}

\def\E{\mathbb{E}}
\def\mv{~\middle\vert~}

\def\mv{~\middle\vert~}
\title{Dual-Delayed Asynchronous SGD for Arbitrarily Heterogeneous Data}

\author{Xiaolu Wang\thanks{Department of Electronic and Computer Engineering, The Hong Kong University of Science and Technology (Emails: xwangcu@gmail.com, yuchang.sun@connect.ust.hk, eejzhang@.ust.hk).} \and Yuchang Sun$^\ast$ \and Hoi-To Wai\thanks{Department of Systems Engineering and Engineering Management, The Chinese University of Hong Kong (Email: htwai@se.cuhk.edu.hk).} \and Jun Zhang$^\ast$}

\date{}

\begin{document}
	
	\maketitle
	
	\begin{abstract}
		We consider the distributed learning problem with data dispersed across multiple workers under the orchestration of a central server. Asynchronous Stochastic Gradient Descent (SGD) has been widely explored in such a setting to reduce the synchronization overhead associated with parallelization. However, the performance of asynchronous SGD algorithms often depends on a bounded dissimilarity condition among the workers' local data, a condition that can drastically affect their efficiency when the workers' data are highly heterogeneous.
		To overcome this limitation, we introduce the \textit{dual-delayed asynchronous SGD (DuDe-ASGD)} algorithm designed to neutralize the adverse effects of data heterogeneity. DuDe-ASGD makes full use of stale stochastic gradients from all workers during asynchronous training, leading to two distinct time lags in the model parameters and data samples utilized in the server's iterations. Furthermore, by adopting an incremental aggregation strategy, DuDe-ASGD maintains a per-iteration computational cost that is on par with traditional asynchronous SGD algorithms.
		Our analysis demonstrates that DuDe-ASGD achieves a near-minimax-optimal convergence rate for smooth nonconvex problems, even when the data across workers are extremely heterogeneous. Numerical experiments indicate that DuDe-ASGD compares favorably with existing asynchronous and synchronous SGD-based algorithms. 
	\end{abstract}
	
	\section{Introduction}
	\vspace{-0mm}
	In traditional machine learning, training often occurs on a single machine. This approach can be restrictive when handling large datasets or complex models that demand substantial computational resources. Distributed machine learning overcomes this constraint by utilizing multiple machines working in parallel. This method distributes the computational workload and data across several nodes or workers, enabling faster and more scalable training.
	
	We focus on the most common distributed machine learning paradigm, known as the \textit{data parallelism} approach. In this setup, the training data are distributed among multiple workers, with each worker independently conducting computations on its local data.
	As an extension of stochastic gradient descent (SGD) used on a single machine, \textit{synchronous SGD} \citep{cotter2011better,dekel2012optimal,chen2016revisiting,goyal2017accurate} stands as a prominent example of data-parallel training algorithms. In synchronous SGD, the server broadcasts the latest model to all workers, who then simultaneously compute stochastic gradients using their respective datasets. After local computation, these workers send their stochastic gradients back to the central server. The server then aggregates these stochastic gradients and updates the global model accordingly.
	
	However, variations in computation speeds and communication bandwidths across workers, typically due to differences in hardware, are common. In synchronous SGD, this disparity forces all workers to wait for the slowest one to complete its computations before proceeding to the next iteration. This issue, often referred to as the \textit{straggler effect}, leads to significant idle times, severely limiting the efficiency and scalability of the approach.
	To address this problem, {\it asynchronous SGD (ASGD)} algorithms have been extensively studied to mitigate the synchronization overhead among workers. Since nodes operate independently, each can proceed at its own pace without waiting for others. This attribute is especially beneficial in ad-hoc clusters or cloud environments where hardware heterogeneity is prevalent \citep{assran2020advances}.
	
	The primary challenge faced by asynchronous training is that its efficiency can be compromised by \textit{data heterogeneity}. This issue arises because fast workers are able to send more frequent updates to the server, while slower workers contribute less frequently. Consequently, the training process may become biased, as the data from fast and slow workers are not equally represented in the server's model updates. 
	Recent research efforts have addressed the problem of data heterogeneity in ASGD \citep{gao2021provable, mishchenko2022asynchronous, koloskova2022sharper,islamov2024asgrad}. These studies focus on the convergence properties of ASGD under conditions where the dissimilarity of local objective functions is bounded. However, if the local datasets are highly heterogeneous, leading to significant differences in local objective functions, then the convergence performance of these algorithms can be substantially reduced.
	
	\subsection{Our Contributions}\label{sec:contribution}
	\vspace{-0mm}
	In this paper, we tackle the above limitations in existing ASGD algorithms. Our main contributions are summarized as follows:
	\begin{enumerate}[label=\arabic*),leftmargin=14pt]
		\item We propose the dual-delayed ASGD (DuDe-ASGD) algorithm for distributed training, with the following key features:
		\begin{enumerate}[label=$\bullet$,leftmargin=12pt]
			\item DuDe-ASGD aggregates the stochastic gradients from \textit{all} workers, which are computed based on both stale models and stale data samples. This leads to a \textit{dual-delayed} aggregated gradient at the server, contrasting sharply with traditional ASGD algorithms that use stale models but fresh data samples for each iteration.
			\item DuDe-ASGD operates in a \textit{fully asynchronous} manner, meaning that the server updates the global model as soon as it receives a stochastic gradient from any worker, without the need to wait for others.  
			Additionally, DuDe-ASGD can be readily adapted to \textit{semi-asynchronous} implementations, which allows it to balance the advantages of both synchronous and asynchronous training methods, demonstrating its high flexibility.
			\item Although DuDe-ASGD requires aggregation of stochastic gradients from all workers in every iteration, it can be implemented incrementally by storing each worker's latest stochastic gradient, maintaining a per-iteration computational cost comparable to traditional ASGD algorithms.
		\end{enumerate}
		
		\item Through a careful analysis accounting for the time lags inherent in the dual-delayed system, we demonstrate that DuDe-ASGD achieves a near-minimax-optimal convergence rate for solving general nonconvex problems under mild assumptions. Our theoretical results do not depend on bounded function dissimilarity conditions, indicating that DuDe-ASGD can achieve rapid and consistent convergence on arbitrarily heterogeneous data.
		
		\item We conduct experiments comparing DuDe-ASGD with other ASGD and aggregation-based algorithms in training deep neural networks on the CIFAR-10 dataset. We show that DuDe-ASGD delivers competitive runtime performance relative to asynchronous and synchronous SGD-based algorithms, validating its effectiveness and efficiency in practical applications.
	\end{enumerate}
	
	\section{Problem Setup and Prior Art}
	\vspace{-0mm}
	Distributed machine learning involving $n$ workers and a server can be described by the following stochastic optimization problem:
	\begin{align}
		\min_{\bw\in\mathbb{R}^p}~
		& F(\bw) \coloneqq \frac{1}{n}\sum_{i=1}^{n}F_i(\bw),~~\text{where}~~F_i(\bw) := \mathbb{E}_{\bxi_i \sim \mathbb{P}_i} \left[ f_i (\bw; \bxi_i) \right].
		\label{eq:opt}
	\end{align}
	Here, $p$ denotes the dimension of the model parameters, and $\bxi_i$ is a data sample from worker $i$, following a probability distribution $\mathbb{P}_i$. Each local loss function $f_i(\cdot; \bxi_i)$, defined for $i \in [n]$ and $\bxi_i \in \Xi_i$, is continuously differentiable and accessible to worker $i$.
	Problem \eqref{eq:opt} shall be solved collaboratively by $n$ workers under the coordination of a central server. 
	Our focus is on scenarios with data heterogeneity, where the local distributions $\mathbb{P}_i$ differ significantly from one another. This setting is particularly relevant in contexts such as data-parallel distributed training \citep{verbraeken2020survey} and horizontal federated learning \citep{yang2019federated}. 
	
	In vanilla ASGD \citep{nedic2001distributed,agarwal2011distributed}, every computed stochastic gradient at a worker triggers a global update at the server, which results in the following iteration formula:
	\begin{align}
		\bw^t = \bw^{t-1} - \eta \nabla f_{j_t} (\bw^{t-\tau_{j_t}(t)};\bxi_{j_t}^t), ~t = 1, 2, \dots,
		\label{eq:asgd}
	\end{align}
	where $j_t \in [n]$ denotes the index of the worker that contributes to the server's iteration $t$ and $\tau_i(t) \in [1,t]$ represents the delay of the model used to compute the stochastic gradient by worker $i$ at server iteration $t$. The updated model, $\bw^t$, is then transmitted back to worker $j_t$ for subsequent local computations. It is important to note that $\bxi_i^t \sim \mathbb{P}_i$ is indexed by $t$ to indicate that this particular data sample has not been utilized by the server prior to iteration $t$. 
	
	The iterative process (\ref{eq:asgd}) allows faster workers to contribute more frequently to the server's model updates. 
	However, when dealing with data heterogeneity where $F_i$ are different, the stochastic gradient $\nabla f_{j_t} (\bw^{t-\tau_{j_t}(t)};\bxi_i^t)$ can significantly deviate from $\nabla F(\bw^t)$ on average, which can impede the model's convergence. 
	To be more specific, we assume that $j_t$ follows some distribution $\{p_1,\dots,p_n\}$ over $[n]$, where $p_i$ is the probability that $j_t = i$ for $i \in [n]$. 
	Even in scenarios without delays in the iterations, i.e., $\tau_i(t) = 1$ for all $i \in [n]$, the stochastic gradient remains a \textit{biased estimate} of the exact gradient:
	\begin{align}
		\mathbb{E} \left[ \nabla f_{j_t} (\bw^{t-1};\bxi_{j_t}^t) \right] = \sum_{i=1}^n p_i \nabla F_i( \bw^{t-1} ) \neq \nabla F ( \bw^{t-1} ).
		\nonumber
	\end{align}
	The convergence analysis of vanilla ASGD on heterogeneous data has been attempted by \citep{mishchenko2022asynchronous}. As reported in Table \ref{tab:complexity}, vanilla ASGD \emph{may not converge to a stationary point of \eqref{eq:opt}} and the asymptotic bias is proportional to the level of data heterogeneity.
	
	To address the disparity between fast and slow workers, \cite{koloskova2022sharper} integrates a random worker scheduling scheme within the ASGD framework. In this approach, after executing iteration \eqref{eq:asgd}, the server sends the updated model $\bw^t$ to a worker sampled from the set of all workers \textit{uniformly at random}. This method promotes more uniform contribution of workers and ensures the convergence of the iterates to a stationary point of Problem \eqref{eq:opt}, achieving the best-known convergence rate for ASGD on heterogeneous data.
	However, as data heterogeneity increases, the convergence rate is adversely affected, as detailed in Table \ref{tab:complexity}. Additionally, there is a potential issue with this scheduling method: a worker may be chosen multiple times consecutively before it completes its current tasks, leading to a backlog of models in the worker's buffer. This accumulation can reduce the overall efficiency of the algorithm, as workers may struggle to process a queue of pending models.
	In contrast to strategies that employ uniformly random worker sampling, \cite{leconte2024queuing} introduces a non-uniform worker sampling scheme in ASGD to balance the accumulation of queued tasks among both fast and slow workers. The analysis involves specific assumptions about the processing time distributions, which facilitate the accurate determination of the stationary distribution of the number of tasks currently being processed.
	Additionally, \cite{islamov2024asgrad} has proposed the \textit{Shuffled ASGD}, which shuffles the sampling order of workers after a specified number of iterations. This approach aims to further enhance the fairness and efficiency of task distribution, ensuring that no single worker consistently benefits or suffers from its position in the sampling sequence.
	Nevertheless, these state-of-the-art ASGD methods all require the dissimilarity among local functions $F_i$ to be bounded. Their performance tends to deteriorate in the presence of high data heterogeneity.
	Further discussion on other works related to asynchronous training methods can be found in Appendix \ref{sec:related}.
	
	\begin{algorithm}[t]
		\caption{DuDe-ASGD (fully asynchronous version without mini-batching)}
		\label{algo1}
		\begin{algorithmic}[1] 
			\STATE \textbf{Input:} $n, T \in \mathbb{Z}_{++}$, $\eta > 0$, $\bw^0 \in \mathbb{R}^p$
			\STATE \textbf{Initialization:} For worker $i \in [n]$, it computes $\nabla f_i(\bw^0,\bxi_i^1)$, stores it in the worker's buffer $\widetilde{\bG}_i$, and sends it to the server. The server computes $\bg^1 = \frac{1}{n}\sum_{i=1}^{n} \widetilde{\bG}_i$ and $\bw^1 = \bw^0 - \eta \bg^1$, store them in the server's buffers $\widetilde{\bg}$ and $\widetilde{\bw}$, and broadcast $\bw^1$ to all workers
			\FOR{$t=2,3,\dots,T$}
			\STATE Once some worker $j_t$ finishes computing $\bG_{j_t}^t \coloneqq \nabla f_{j_t} ( \bw^{t-\tau_{j_t}(t)}; \bxi_{j_t}^t )$, it sends $\bm{\delta}^t \coloneqq \bG_{j_t}^t - \widetilde{\bG}_{j_t}$ to the server and update its local buffer $\widetilde{\bG}_{j_t} \leftarrow \bG_{j_t}$
			\STATE The server computes the aggregated gradient as
			$
			\bg^t = \widetilde{\bg} + \bm{\delta}^t / n
			$
			and updates the server's buffer $\widetilde{\bg} \leftarrow \bg^t$
			\STATE The server computes the new model as $\bw^t = \widetilde{\bw} - \eta \bg^t$, sends $\bw^t$ to worker $j_t$, and updates the server's buffer $\widetilde{\bw} \leftarrow \bw^t$
			\ENDFOR
			\STATE \textbf{Output:} $\bw^t$, where $t$ selected uniformly random from $[T]$
		\end{algorithmic}
	\end{algorithm}
	
	\section{Dual-Delayed ASGD (DuDe-ASGD)}\label{sec:algorithm}
	\vspace{-0mm}
	Given the challenges of managing data heterogeneity while ensuring rapid convergence in ASGD, we introduce the Dual-Delayed ASGD (DuDe-ASGD). This method enhances updates by employing the \textit{full aggregation} technique, which utilizes not just the stochastic gradient from a single worker, but also stale stochastic gradients from all other workers. The update formula for DuDe-ASGD is
	\begin{align}\label{eq:iter}
		\bw^t = \bw^{t-1} - \eta \underbrace{\frac{1}{n} \sum_{i=1}^{n} \nabla f_i ( \bw^{t-\tau_i(t)}; \bxi_i^{t-d_i(t)} )}_{\bg^t}, ~t = 1, 2, \dots,
	\end{align}
	where $d_i(t) \in [0, t-1]$ is the delay of the data sample that worker $i$ used to compute the stochastic gradient in server's iteration $t$, and $\bg^t$ is the aggregated gradient that involves the most recent computation results of all workers. 
	Unlike the traditional ASGD \eqref{eq:asgd} where only the model experiences delay while the data sample remains current, DuDe-ASGD involves two distinct types of delays---in both the models and the data samples---in the updates, hence termed \textit{dual-delayed}.
	
	The dual-delay property emerges not from intentional design, but as a natural consequence of integrating asynchronous training with full aggregation.
	We initialize $\tau_i(1) = 1$ and $d_i(1) = 0$ for all $i \in [n]$, then the evolution of the delays associated with the data samples is described by:
	\begin{align}
		d_i(t) = 
		\begin{cases}
			0, &\text{if}~i = j_t
			\\
			d_i(t-1) + 1, &\text{if}~i \neq j_t
		\end{cases},
		~t = 2, 3, \dots.
		\nonumber
	\end{align}
	Notably, the aggregated gradient $\bg^t$ utilized by DuDe-ASGD allows for incremental updates:
	\begin{align}
		\bg^t = \bg^{t-1} - \frac{1}{n} \nabla f_{j_t} ( \bw^{t-1-\tau_{j_t}(t-1)}; \bxi_{j_t}^{t-1-d_{j_t}(t-1)} ) + \frac{1}{n} \nabla f_{j_t} ( \bw^{t-\tau_{j_t}(t)}; \bxi_{j_t}^{t} ),
		\nonumber
	\end{align}
	whose computational cost per iteration is independent of the number of workers $n$. This efficiency is achieved by maintaining a record of the latest stochastic gradients from all $n$ workers, so that the per-iteration computational complexity of DuDe-ASGD aligns with that of traditional ASGD.
	For every newly computed stochastic gradient, the model on which they are computed can be delayed while the data is freshly sampled, ensuring that the delays in the models always surpass those in the data samples within the aggregated gradient $\bg^t$, i.e., for all $i \in [n]$ and $t\geq 1$,
	\begin{align}\label{eq:tau-d}
		\tau_i(t) \geq d_i(t) + 1.
	\end{align}
	The training procedures for DuDe-ASGD are described in Algorithm \ref{algo1}.
	The distinctions between traditional ASGD and DuDe-ASGD during a single communication round are illustrated in Figure \ref{fig:asgd-illustration}.
	In Algorithm \ref{algo1}, each worker computes a stochastic gradient using just one data sample at a time. 
	However, to better balance the stochastic gradient noise and per-iteration computation/memory costs, it is advantageous to employ multiple samples simultaneously. 
	With a slight abuse of notation, we define $\nabla f_i ( \bw^{t-\tau_i(t)}; \mathcal{D}_i^{t-d_i(t)} )
	\coloneqq
	\frac{1}{b_i} \sum_{k=1}^{b_i} \nabla f_i ( \bw^{t-\tau_i(t)}; \bxi_{i,k}^{t-d_i(t)} ),$
	where $\mathcal{D}_i$ is a set of data samples independently drawn from $\mathbb{P}_i$ with batch size $b_i \geq 1$.
	Using $\bg^t = \frac{1}{n} \sum_{i=1}^{n} \nabla f_i ( \bw^{t-\tau_i(t)}; \mathcal{D}_i^{t-d_i(t)} )$
	as the aggregated stochastic gradient in iteration \eqref{eq:iter} yields the \textit{mini-batch} version of DuDe-ASGD that is useful in practice.
	
	\textbf{Semi-Asynchronous Variant.}
	To balance the strengths and weaknesses of asynchronous and synchronous strategies, adopting ``hybrid'' approaches can be beneficial.
	On one side, fully synchronous algorithms, such as synchronous SGD, can be significantly hindered by stragglers during the training process.
	On the other side, in fully asynchronous algorithms, such as the one detailed in Algorithm \ref{algo1}, each stochastic gradient computation immediately triggers a global update at the server. While this method can reduce wait times and potentially increase throughput, it may also result in high levels of staleness in the models and the data samples within $\bg^t$, which may adversely affect the overall time efficiency.
	Considering these factors, DuDe-ASGD can be adapted to a \textit{semi-asynchronous variant}, in which the server waits for stochastic gradients from multiple workers before performing a new update.
	Specifically, we define $\mathcal{C}_t \coloneqq \{i: d_i(t) = 0\}$ as the set of workers that contribute to the server's iteration $t$. 
	The semi-asynchronous variant, incorporating the mini-batch implementation, then updates the model using the aggregated stochastic gradient that is computed in the following manner:
	\begin{align}
		\bg^t 
		&= \bg^{t-1} - \frac{1}{n} \sum_{i \in \mathcal{C}_t} \nabla f_i ( \bw^{t-1-\tau_i(t-1)}; \mathcal{D}_i^{t-1-d_i(t-1)} ) 
		+ \frac{1}{n} \sum_{i \in \mathcal{C}_t} \nabla f_i ( \bw^{t-\tau_i(t)}; \mathcal{D}_i^{t} ).
		\nonumber
	\end{align}
	In particular, when $\tau_i(t) = d_i(t)+1$ for all $i \in [n]$, then $\bg^t = \frac{1}{n} \sum_{i=1}^{n} \nabla f_i ( \bw^{t-\tau_i(t)}; \bxi_i^{t-\tau_i(t)+1} )$, which aligns with the approaches used in MIFA \citep{gu2021fast} (without multiple local updates) and sIAG \citep{wang2023linear}. 
	These two algorithms differ essentially from DuDe-ASGD in that the delays associated with the model parameters and data samples are synchronized, categorizing them as synchronous algorithms. 
	Moreover, if $\tau_i(t) = 1$ for all $i \in [n]$, then DuDe-ASGD becomes equivalent to synchronous SGD.
	
	\begin{figure}
		\centering
		\begin{subfigure}[b]{0.45\textwidth}
			\centering
			\includegraphics[width=\textwidth]{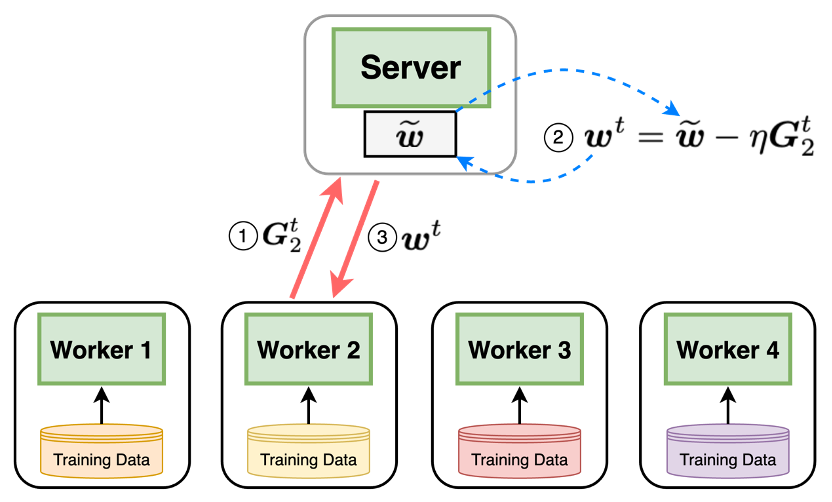}
			\caption{Traditional ASGD}
			\label{fig:asgd}
		\end{subfigure}
		\hspace{6mm}
		\begin{subfigure}[b]{0.45\textwidth}
			\centering
			\includegraphics[width=\textwidth]{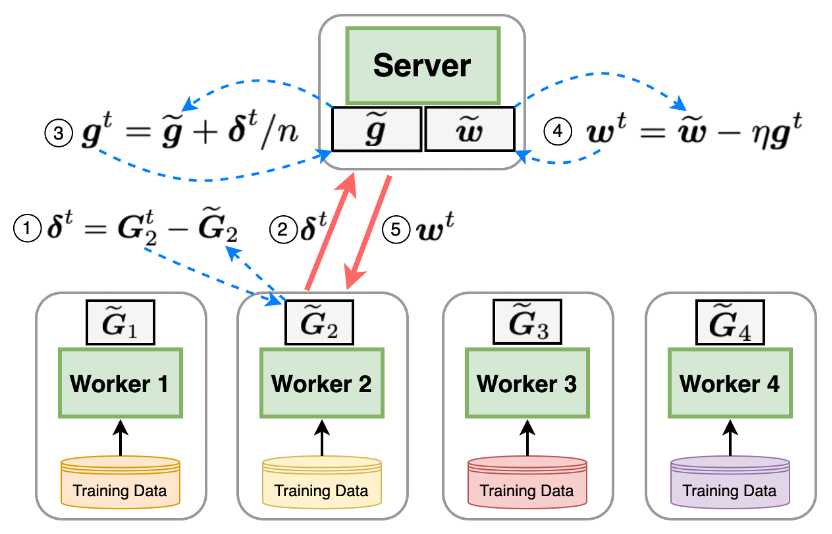}
			\caption{DuDe-ASGD}
			\label{fig:asgd2}
		\end{subfigure}
		\caption{Illustration of traditional ASGD and the proposed DuDe-ASGD. Suppose that worker 2 contributes to the server's model update in iteration $t$. In traditional ASGD algorithms, each worker directly sends the freshly computed stochastic gradient ${\bG}^t_2 = \nabla f_2 (\bw^{t-\tau_2(t)};\bxi_2^t)$ to the server. While in DuDe-ASGD, each worker maintains a memory of the most recently evaluated stochastic gradient $\widetilde{\bG}_2$ and sends the gradient difference $\bm{\delta}^t = {\bG}^t_2 - \widetilde{\bG}_2$.}
		\label{fig:asgd-illustration}
	\end{figure}
	
	\section{Theoretical Analysis}\label{sec:theory}
	\vspace{-0mm}
	This section delves into the theoretical underpinnings of convergence behaviors of DuDe-ASGD. 
	For clarity in our discussion, we present the convergence analysis of Algorithm \ref{algo1} in its fully asynchronous form. The technical results discussed herein can be readily adapted to the semi-asynchronous variant of DuDe-ASGD, with or without the implementation of mini-batching.
	
	To proceed, we introduce the following standard assumptions that are essential in our analysis.
	\begin{assumption}\label{as:lowerbound}
		There exists $F^* > -\infty$ such that (s.t.) $F(\bw) \geq F^*$ for all $\bw \in \mathbb{R}^p$.
	\end{assumption}
	
	\begin{assumption}\label{as:function}
		$F$ is $L$-smooth, i.e., $F$ is continuously differentiable and there exists $L \geq 0$ s.t.
		\begin{align}
			\| \nabla F(\bw) - \nabla F(\bw') \|_2 \leq L \| \bw - \bw' \|_2,~~\forall~\bw,\bw' \in \mathbb{R}^p.
			\nonumber
		\end{align}
	\end{assumption}
	
	\begin{assumption}\label{as:unbiased}
		Let $\bw^r$ and $\bw^s$ with $s,r \geq 0$ be iterates generated by Algorithm \ref{algo1}, $\bxi_i^t \in \Xi_i$ with $i \in [n]$ and $t \geq 1$ be a data sample drawn from $\mathbb{P}_i$, and ${\cal F}_s$ be the sigma algebra generated by $\bw^1,\dots,\bw^s$. 
		If $r \leq s< t$, then
		\begin{align}
			\E \left[ \nabla f_i(\bw^r;\bxi_i^t) \mid {\cal F}_s \right] = \nabla F_i(\bw^r).
			\label{eq:unbiased}
		\end{align}
	\end{assumption}
	Assumption \ref{as:unbiased} specifies that the stochastic gradient estimate is unbiased. It is important to note that the iterate $\bw^r$ may depend on $\bxi_i^t$ for different times $r$ and $t$, rendering $\nabla f_i (\bw^r;\bxi_i^t)$ a potentially biased estimate of $\nabla F_i(\bw^r)$. To maintain the unbiasedness as defined in equation \eqref{eq:unbiased}, it is critical to ensure that $r \leq s < t$. This condition guarantees that ${\cal F}_s$ encompasses all information present in $\bw^r$ and that $\bxi_i^t$ is independent of ${\cal F}_s$. Furthermore, we impose upper bounds on both the conditional variance of stochastic gradients:
	\begin{assumption}\label{as:sigma}
		Let $\bw^r$ and $\bw^s$ with $s,r \geq 0$ be iterates generated by Algorithm \ref{algo1}, $\bxi_i^t \in \Xi_i$ with $i \in [n]$ and $t \geq 1$ be a data sample drawn from $\mathbb{P}_i$, and ${\cal F}_s$ be the sigma algebra generated by $\bw^1,\dots,\bw^s$. 
		If $r \leq s< t$, then there exists a constant $\sigma \geq 0$ s.t.
		\begin{align}
			\E \left[ \| \nabla f_i(\bw^r;\bxi_i^t) - \nabla F_i(\bw^r) \|_2^2 \mid {\cal F}_s \right] \leq \sigma^2.
			\nonumber
		\end{align}
	\end{assumption}
	We further assume that each worker contributes to the server's updates within a bounded number of global iterations, encapsulated by the following assumption regarding the maximum delay of model parameters:
	\begin{assumption}\label{as:delay}
		There exists $\tau_{\max} \geq 1$ s.t. $\tau_i(t) \leq \tau_{\max}$ for all $i \in [n]$ in Algorithm \ref{algo1}.
	\end{assumption}
	When $\tau_{\max} = 1$, indicating that all the models within $\bg^t$ used in iteration \eqref{eq:iter} is up-to-date, Algorithm \ref{algo1} simplifies to synchronous SGD.
	In the context of semi-asynchronous DuDe-ASGD where $|\mathcal{C}_t| = c \in [n]$ for all $t \geq 1$, we use $\tau_{\max}^{(c)}$ to denote the maximum model delay. Then, the relationship between $\tau_{\max}$ and $\tau_{\max}^{(c)}$ is characterized by $\tau_{\max}^{(c)} = \tau_{\max}/c$ \citep[Appendix A]{nguyen2022federated}. This formula signifies that the maximum model delay decreases proportionally to the number of workers the server waits for in each iteration.
	The technical results for fully asynchronous DuDe-ASGD (Algorithm \ref{algo1}) can be readily adapted to the semi-asynchronous variant by substituting $\tau_{\max}$ with $\tau_{\max}^{(c)}$ in our analysis.
	
	\subsection{Convergence Analysis of DuDe-ASGD}\label{sec:rate}
	\vspace{-0mm}
	To study the convergence of DuDe-ASGD, we first observe that under Assumption~\ref{as:function}, the following descent lemma holds:
	\begin{align}
		\E[F(\bw^t)]-\E[F(\bw^{t-1})] 
		\leq&~ \E [ \langle \nabla F(\bw^{t-1}), \bw^t - \bw^{t-1} \rangle ] + \frac{L}{2} \E \| \bw^t - \bw^{t-1} \|_2^2
		\nonumber
		\\
		=& - \eta \E \langle \nabla F(\bw^{t-1}), \bg^t \rangle + \frac{L\eta^2}{2} \E \| \bg^t \|_2^2.
		\label{eq:descent}
	\end{align}
	Intuitively, both $\E \langle \nabla F(\bw^{t-1}), \bg^t \rangle$ and $\E \| \bg^t \|_2^2$ can be regarded as biased estimates of $\| \nabla F(\bw^{t-1}) \|_2^2 \geq 0$. Subsequently, selecting an appropriate step size $\eta$ can make the right-hand side of \eqref{eq:descent} negative, thereby ensuring a sufficient decrease in the expected function value in each iteration. 
	
	However, due to the dual delay properties of the information encapsulated in $\bg^t$, handling the inner product term is not trivial.
	To describe the challenge, note that 
	\[ \textstyle
	\langle \nabla F(\bw^{t-1}), \bg^t \rangle = \frac{1}{n} \sum_{i=1}^{n} \langle \nabla F(\bw^{t-1}), \nabla f_i ( \bw^{t-\tau_i(t)}; \bxi_i^{t-d_i(t)} ) \rangle.
	\]
	We observe that $\bw^{t-1}$ is a function of $\bxi_i^{t-d_i(t)}$ for $i \in [n]$ such that $d_i(t) \geq 1$.
	Consequently, 
	\[
	\E \langle \nabla F(\bw^{t-1}), \nabla f_i ( \bw^{t-\tau_i(t)}; \bxi_i^{t-d_i(t)} ) \rangle
	\neq \E \langle \nabla F(\bw^{t-1}), \nabla F_i ( \bw^{t-\tau_i(t)}) \rangle
	\]
	and we can no longer obtain a simple expression for the expectation of the inner product. \footnote{We remark that some prior studies \citep{lian2018asynchronous,avdiukhin2021federated,zhang2023no,wang2023tackling} involve similar inner product terms in analysis and have treated them with $\E \left\langle \nabla F(\bw^r), \nabla f_i ( \bw^s; \bxi_i^t ) \right\rangle = \E \left\langle \nabla F(\bw^r), \nabla F_i ( \bw^s) \right\rangle$ for $s < t$. This nuanced but crucial issue may not have been adequately emphasized in these works.} To address this challenge, our idea is to decompose the inner product into two terms:
	\begin{align}
		\langle \nabla F(\bw^{t-1}), \bg^t \rangle
		= \langle \nabla F(\bw^{[t-\tau_{\max}]_+}), \bg^t \rangle
		+ \langle \nabla F (\bw^{t-1}) - \nabla F(\bw^{[t-\tau_{\max}]_+}), \bg^t \rangle,
		\label{eq:decomp}
	\end{align}
	where $[x]_+ \coloneqq \max\{x,0\}$ for $x \in \mathbb{R}$.
	Then utilizing Assumption \ref{as:delay} and considering the expectation conditioned on the most outdated model, we have the desired property:
	\[
	\textstyle \E \left\langle \nabla F(\bw^{[t-\tau_{\max}]_+}), \bg^t \right\rangle
	= \frac{1}{n} \sum_{i=1}^n \E \langle \nabla F(\bw^{[t-\tau_{\max}]_+}) , \nabla F_i ( \bw^{t-\tau_i(t)} ) \rangle
	\]
	Through carefully controlling the second error term, we arrive at the following lower bound on \eqref{eq:decomp}:
	
	\begin{proposition}\label{prop:cc}
		Suppose that Assumptions \ref{as:function}--\ref{as:delay} hold. If the stepsize satisfies $\eta \leq \frac{1}{16 L \tau_{\max}}$,
		then it holds for all $t \geq 1$ that
		\begin{align}
			&\E \langle \nabla F(\bw^{t-1}), \bg^t \rangle
			\geq \frac{1}{8} \E \| \nabla F (\bw^{t-1}) \|_2^2 
			- 2 L \tau_{\max} \eta \sum_{s=1+[t-\tau_{\max}]_+}^{t} \E \| \nabla F(\bw^{s-1}) \|_2^2 
			\nonumber
			\\
			&\qquad\qquad\quad 
			- 3 L \tau_{\max} \eta \frac{\sigma^2}{n}
			- 6 L^2 \tau_{\max} \eta^2 \sum_{s=1+[t-\tau_{\max}]_+}^{t} \E \left\| \frac{1}{n} \sum_{j=1}^{n} \nabla F_j (\bw^{s-\tau_j(s)}) \right\|_2^2.
			\nonumber
		\end{align}
	\end{proposition}
	The proof of Proposition \ref{prop:cc} is deferred to Appendix \ref{appen:prop}.
	Equipped with this, we finally establish the convergence rate of DuDe-ASGD by choosing proper step size $\eta$, as stated in the following theorem:
	\begin{theorem}\label{thm}
		Suppose that Assumptions \ref{as:lowerbound}--\ref{as:delay} hold. 
		Let $\{\bw^t\}_{t=1}^{T}$ be the sequence generated by Algorithm \ref{algo1} and the step size be $\eta = \frac{1}{2} \sqrt{\frac{n\Delta}{L \sigma^2 \tau_{\max} T}}$ with $\Delta \coloneqq F(\bw^0) - F^*$. Then, for $T \geq \frac{1024 L \Delta n \tau_{\max}}{\sigma^2}$, it holds that
		\begin{align}
			\frac{1}{T} \sum_{t=1}^{T} \E \| \nabla F (\bw^{t-1}) \|_2^2  
			\leq 128 \sqrt{\frac{L \Delta \sigma^2  \tau_{\max} }{nT}} +  \frac{128(L\Delta)^{3/2} \sqrt{n\tau_{\max}}}{\sigma   T^{3/2}}.
			\nonumber
		\end{align}
	\end{theorem}
	The proof of Theorem \ref{thm} is deferred to Appendix \ref{appen:thm}.
	Theorem \ref{thm} demonstrates that DuDe-ASGD converges to a stationary point of Problem \eqref{eq:opt} at a rate of $\mathcal{O}\left( \sqrt{ {L \Delta \sigma^2  \tau_{\max} } / {nT}} \right)$ and the \textit{transient time} required for convergence exhibits a moderate linear dependence on both the number of workers $n$ and the maximum model delay $\tau_{\max}$.
	Critically, the convergence rate of DuDe-ASGD is achieved without imposing any assumptions on upper bounds for data heterogeneity or the dissimilarity among individual functions $F_i$. This indicates that DuDe-ASGD is well-suited for distributed environments with highly heterogeneous data.
	For sufficiently small $\epsilon > 0$, we can deduce that after acquiring
	\begin{equation} \label{eq:sample_complex} 
		\mathcal{O}\left( \frac{ L \Delta \sigma^2 \tau_{\max}}{n \epsilon^2}
		\right)~~\text{samples},
	\end{equation}
	the iterates produced by Algorithm \ref{algo1} satisfies $\frac{1}{T} \sum_{t=1}^{T} \E \| \nabla F (\bw^{t-1}) \|_2^2 \leq \epsilon$.
	This sample complexity aligns closely with the theoretical lower bound for finding $\epsilon$-stationary points using stochastic first-order methods \citep{arjevani2023lower}:\footnote{There exists a problem instance satisfying Assumptions \ref{as:lowerbound}--\ref{as:sigma} under which every randomized algorithm requires at least $c (L\Delta\sigma/\epsilon^2 + L\Delta/\epsilon)$ samples to find $\bw$ s.t. $\E \|\bw\|_2^2 \leq \epsilon$, where $c > 0$ is a universal constant.}
	\begin{corollary}\label{col}
		Provided that each worker contributes to the server's updates just once every $n$ iterations, which implies that  $\tau_{\max} = n$, then the sample complexity \eqref{eq:sample_complex} of Algorithm \ref{algo1} is minimax optimal.
	\end{corollary}
	
	\subsection{Comparisons with Prior Works}\label{sec:comparisions}
	\vspace{-0mm}
	We compare the theoretical performance of DuDe-ASGD with several representative distributed SGD-based algorithms as detailed in Table \ref{tab:complexity}.
	
	\textbf{Aggregation-Based Algorithms.}
	Representative SGD-based algorithms that employ full aggregation strategies include sIAG \citep{wang2023linear} and MIFA \citep{gu2021fast}, both of which operate synchronously.
	The analysis of sIAG is applicable only to strongly convex objectives. The convergence rate of MIFA is established based on the Lipschitz continuity of Hessians and boundedness of gradient noise, with the transient time $\Omega(n\tau_{\max}^2)$ exhibiting a quadratic dependence on the maximum model delay. By contrast, our analysis is conducted under less restrictive conditions and achieves a reduced transient time.
	FedBuff \citep{nguyen2022federated}, a notable semi-asynchronous federated learning algorithm, incorporates partial aggregation where only a subset of delayed local updates are considered during each global update.
	There has been several convergence analyses for FedBuff \citep{nguyen2022federated,toghani2022unbounded,wang2023tackling}, while they all assume equal probability of worker contributions to the global update---an idealistic scenario that rarely holds in practice.
	
	\textbf{Asynchronous Algorithms.}
	The ASGD algorithms \citep{mishchenko2022asynchronous,koloskova2022sharper,islamov2024asgrad} require bounded data heterogeneity---an assumption not necessary in our analysis for DuDe-ASGD.
	In addition, through employing the full aggregation technique, the dominant term $\sqrt{1/nT}$ in the convergence rate of DuDe-ASGD suggests a \textit{linear speedup} relative to the number of workers $n$, a feature not achieved in prior ASGD variants.	
	
	\renewcommand{\thefootnote}{\arabic{footnote}}
	\begin{table}[t]
		\fontsize{8}{17}\selectfont
		\setlength\tabcolsep{4.3pt}
		\centering
		\caption{Convergence rates of representative distributed SGD-based algorithms for solving smooth nonconvex objectives with heterogeneous data.
			(Shorthand notation: 
			\textbf{Async.} = Asynchronous, 
			\textbf{Agg.} = Aggregation-based, 
			\textbf{Add. Assump.} = Additional assumptions aside from Assumptions \ref{as:lowerbound}--\ref{as:sigma}, 
			{\sf BDH} = Bounded Data Heterogeneity\tablefootnote{There exists $\zeta_i>0$ s.t. $\| \nabla F_i(\bw) - \nabla F(\bw) \|_2^2 \leq \zeta_i^2$ for all $i \in [n]$ and $\bw \in \mathbb{R}^p$. Define $\zeta^2 \coloneqq \frac{1}{n} \sum_{i=1}^{n} \zeta_i^2$ and $\zeta_{\max} \coloneqq \max_{i \in [n]} \{ \zeta_i \}$, which characterize the heterogeneity of data distributions.}, 
			\textsf{BN} = Bounded Noise\tablefootnote{There exists $\delta>0$ s.t. $\|\nabla f_i(\bw;\bxi_i)-\nabla F_i(\bw)\|_2 \leq \delta$ almost surely for all $i \in [n]$, $\bw\in\mathbb{R}^p$, and $\bxi_i \sim \mathbb{P}_i$.}, 
			\textsf{LH} = Lipschitz Hessian\tablefootnote{There exists $\rho > 0$ s.t. $\| \nabla^2 F_i(\bw) - \nabla^2 F_i(\bw') \|_2 \leq \rho \| \bw - \bw' \|_2$ for all $\bw,\bw' \in \mathbb{R}^p$.}, 
			\textsf{UWP} = Uniform Worker Participation\tablefootnote{Every worker contributes to the global aggregation with equal probability.}, 
			\textsf{BG} = Bounded Gradients\tablefootnote{There exists $G \geq 0$ s.t. $\|\nabla F_i(\bw)\|_2^2 \leq G^2$ for all $i \in [n]$.})
		}
		\vspace{1mm}
		\begin{tabular}{c c c c c}
			\toprule 
			\textbf{Algorithms}
			&
			\textbf{Async.?}
			& 
			\textbf{Agg.?}
			& 
			\textbf{Convergence Rates}
			&
			\textbf{Add. Assump.}
			\\
			\midrule
			\begin{tabular}{@{}c@{}}Synchronous SGD \vspace{-2mm}\\ \citep{khaled2023better}
			\end{tabular}
			& No
			& \textbf{Yes}
			& {\scriptsize $\displaystyle \mathcal{O}\left( \sqrt{\frac{\sigma^2}{nT}} + \frac{1}{T} \right)$}
			& --
			\\
			\hline
			\begin{tabular}{@{}c@{}}MIFA \vspace{-2mm}\\  \citep{gu2021fast}
			\end{tabular}
			& No
			& \textbf{Yes}
			& {\scriptsize $\displaystyle \mathcal{O}\left( \sqrt{\frac{1+\tau_{\avg}}{nKT}}\sigma^2 + \frac{n K \sigma \tau_{\max}{\color{magenta}\zeta} + \sigma^2\tau_{\max}{\color{magenta}\delta\rho}}{T} \right)$}\tablefootnote{$K$ is the number of local updates and $\tau_{\avg} \coloneqq \frac{1}{n(T-1)}\sum_{t=1}^{T-1}\sum_{i=1}^{n} \tau_i(t)$.}
			& {\sf BDH}, \textsf{BN}, \textsf{LH}
			\\
			\hline
			\begin{tabular}{@{}c@{}}FedBuff\vspace{-2mm} \\  \citep{nguyen2022federated}
			\end{tabular}
			& Semi
			& Partial
			& 
			{\scriptsize $\displaystyle \mathcal{O}\left( \frac{\sigma^2+K {\color{magenta}\zeta^2}}{\sqrt{mKT}}
				+ \frac{K\tau_{\avg}\tau_{\max}{\color{magenta}\zeta^2}+\tau_{\max}\sigma^2}{T}
				\right)$}\tablefootnote{$m$ is the number of workers participating in each global aggregation. We report the best-known convergence rate of FedBuff established in \citep{wang2023tackling}.
			}
			& {\sf BDH}, \textsf{UWP}
			\\
			\hline
			\begin{tabular}{@{}c@{}}Vanilla ASGD \vspace{-2mm}\\  \citep{mishchenko2022asynchronous}
			\end{tabular}
			& \textbf{Yes}
			& No
			& {\scriptsize $\displaystyle \mathcal{O}\left( \sqrt{\frac{\sigma^2}{T}} + \frac{n}{T} + {\color{magenta}\zeta_{\max}^2} \right)$}
			& {\sf BDH}
			\\
			\hline
			\begin{tabular}{@{}c@{}}Uniform ASGD \vspace{-2mm}\\ \citep{koloskova2022sharper}
			\end{tabular}
			& \textbf{Yes}
			& No
			& 
			{\scriptsize $\displaystyle \mathcal{O}\left( \sqrt{\frac{\sigma^2+{\color{magenta}\zeta^2}}{T}} + \frac{\sqrt[3]{\tau_{\avg}\frac{1}{n}\sum_{i=1}^{n}\tau_{\avg}^i{\color{magenta}\zeta_i^2}}}{T^{2/3}} \right)$}
			& {\sf BDH}
			\\
			\hline
			\begin{tabular}{@{}c@{}}Shuffled ASGD \vspace{-2mm}\\ \citep{islamov2024asgrad}
			\end{tabular} 
			& \textbf{Yes}
			& No
			& {\scriptsize $\displaystyle \mathcal{O}\left( \sqrt{\frac{\sigma^2}{T}} + \frac{(\sqrt{n}{\color{magenta}\zeta})^{2/3}+(n{\color{magenta}G})^{2/3}}{T^{2/3}} + \frac{n}{T} \right)$}
			& {\sf BDH}, \textsf{BG}
			\\
			\hline
			\cellcolor{gray!10}
			\begin{tabular}{@{}c@{}}DuDe-ASGD\vspace{-2mm} \\ (This Paper)
			\end{tabular}   
			& {\cellcolor{gray!10}\textbf{Yes}}
			& {\cellcolor{gray!10}\textbf{Yes}}
			& \cellcolor{gray!10} {\scriptsize $\displaystyle \mathcal{O}\left( \sqrt{\frac{\sigma^2  \tau_{\max} }{nT}} +  \frac{\sqrt{n\tau_{\max}}}{\sigma T^{3/2}} \right)$}
			& {\cellcolor{gray!10} --}
			\\
			\bottomrule
		\end{tabular}
		\label{tab:complexity}
		\vspace{-0mm}
	\end{table}
	
	\section{Numerical Experiments}\label{sec:numerical}
	\vspace{-0mm}
	\textbf{Experiment Setup.} 
	We simulate a distributed system comprising $n$ workers. 
	To model the hardware variations across different workers, we employ the \textit{fixed-computation-speed model} described in \citep{mishchenko2022asynchronous}. Specifically, each worker $i$ consistently takes fixed units of time, $s_i$, to compute a stochastic gradient. For each $i\in [n]$, $s_i$ is drawn from the truncated normal distribution $\mathcal{TN}(\mu, \texttt{std})$ with a mean $\mu = 1$ and standard deviation $\texttt{std} = 1$ and $5$, ensuring all time values are greater than $0$. A higher $\texttt{std}$ indicates more significant hardware variation, leading to a greater maximum delay in the models during the training process. Furthermore, we assume that the communication time between the server and workers, as well as the server's computation time for global updates, are negligible.
	We implement DuDe-ASGD in its fully asynchronous form using mini-batching, along with other distributed SGD-based algorithms listed in Table \ref{tab:complexity}. 
	Each mini-batch comprises 64 samples, uniformly drawn from the local datasets allocated to the workers.
	We evaluate the performance of these algorithms using the CIFAR-10 image dataset \citep{krizhevsky2009learning} by training a convolutional neural network with two convolutional layers for image classification. 
	Following the approach described in \cite{yurochkin2019bayesian}, we allocate the dataset to the workers based on the Dirichlet distribution with concentration parameter $\alpha$. A lower $\alpha$ results in greater data heterogeneity among the workers.
	The step sizes for the algorithms under comparison are selected from the set $\{0.001, 0.005, 0.01\}$, based on which they achieve the fastest convergence. 
	We implement all algorithms in PyTorch and run experiments on NVIDIA RTX A5000 GPUs.
	
	\textbf{Numerical Results.}
	Each experiment is independently repeated three times using different random seeds, and the mean and standard deviation of the numerical performance for a configuration of $n=10$ workers are shown in Figure \ref{fig:loss}.
	In scenarios of high data heterogeneity, specifically $\alpha=0.1$, DuDe-ASGD demonstrates superior performance by achieving the fastest convergence rate in training loss and the highest test accuracy. 
	Additionally, DuDe-ASGD maintains consistent performance even as the computation speeds of the workers vary significantly, as indicated by an increasing \texttt{std}, indicating its robustness to hardware variations. 
	On the other hand, under conditions of low data heterogeneity, where $\alpha=0.5$, the performance of DuDe-ASGD aligns closely with that of vanilla ASGD. 
	This similarity supports the theoretical convergence rate of vanilla ASGD, which includes an additive $\zeta_{\max}^2$ term that becomes less significant as $\alpha$ increases. 
	The convergence rate of synchronous SGD is theoretically invariant across different levels of data heterogeneity. However, its practical runtime performance suffers from the slowest worker, particularly as \texttt{std} increases.
	Furthermore, the Uniform ASGD does not deliver satisfactory outcomes, potentially because the repeated sampling of a slow worker before it completes its task can impair performance. 
	
	\begin{figure}[t]
		\centering
		\begin{subfigure}[b]{\textwidth}
			\centering
			\includegraphics[width=\textwidth]{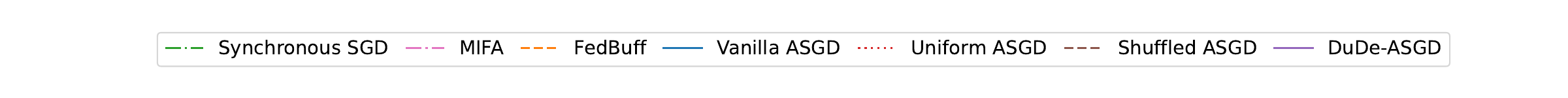}
			\vspace{-7mm}
		\end{subfigure}
		\begin{subfigure}[t]{\textwidth}
			\centering
			\includegraphics[width=\textwidth]{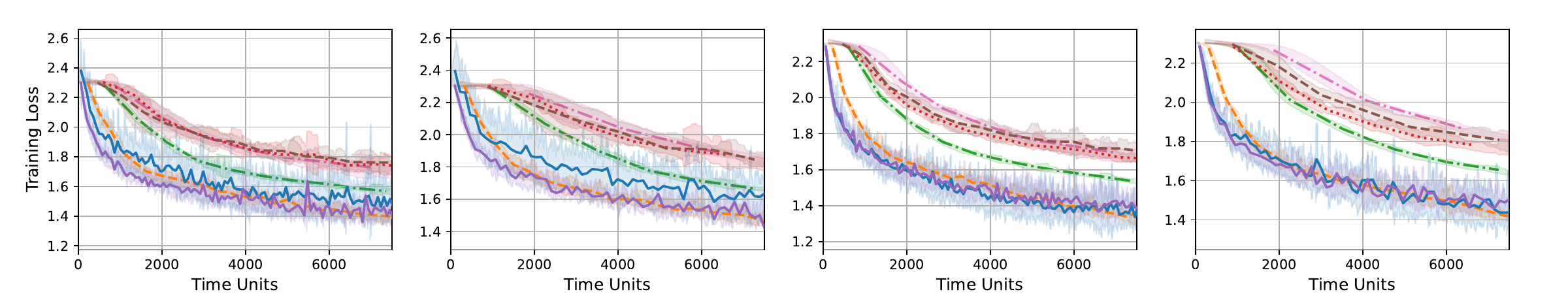}
			\vspace{-7mm}
		\end{subfigure}
		\begin{subfigure}[t]{\textwidth}
			\centering
			\includegraphics[width=\textwidth]{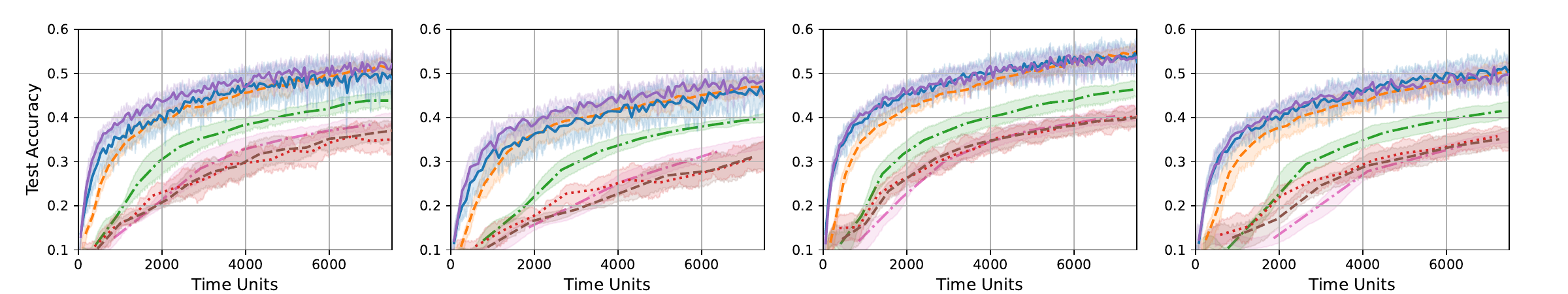}
		\end{subfigure}
		\caption{Convergence curves displaying training losses and test accuracies over time with $n=10$ workers. (1st column: $\alpha=0.1, \texttt{std}=1$; 2nd column: $\alpha=0.1, \texttt{std}=5$; 3rd column: $\alpha=0.5, \texttt{std}=1$; 4th column: $\alpha=0.5, \texttt{std}=5$) 
		}
		\vspace{-0mm}
		\label{fig:loss}
	\end{figure}
	
	\section{Conclusion}\label{sec:conclusion}
	\vspace{-0mm}
	This paper introduces the dual-delayed asynchronous SGD (DuDe-ASGD), a novel approach to distributed machine learning that effectively counteracts the challenges posed by data heterogeneity across workers. By leveraging an asynchronous mechanism that utilizes stale gradients from all workers, DuDe-ASGD not only alleviates synchronization overheads but also balances the contributions of diverse worker datasets to the learning process. Our comprehensive theoretical analysis shows that DuDe-ASGD achieves near-minimax-optimal convergence rates for nonconvex problems, regardless of variations in data distribution across workers. This significant advancement highlights the robustness and efficiency of DuDe-ASGD in handling highly heterogeneous data scenarios, which are prevalent in modern distributed environments. Experiments on real-world datasets validate its superior runtime performance under high data heterogeneity in comparison to other leading distributed SGD-based algorithms, thereby confirming its potential as an effective tool for large-scale machine learning tasks. 
	
	\bibliographystyle{plainnat}
	\bibliography{asgd}

\begin{thebibliography}{55}
\providecommand{\natexlab}[1]{#1}
\providecommand{\url}[1]{\texttt{#1}}
\expandafter\ifx\csname urlstyle\endcsname\relax
  \providecommand{\doi}[1]{doi: #1}\else
  \providecommand{\doi}{doi: \begingroup \urlstyle{rm}\Url}\fi

\bibitem[Agarwal and Duchi(2011)]{agarwal2011distributed}
Alekh Agarwal and John~C Duchi.
\newblock Distributed delayed stochastic optimization.
\newblock In \emph{Advances in Neural Information Processing Systems 24}, 2011.

\bibitem[Arjevani et~al.(2020)Arjevani, Shamir, and Srebro]{arjevani2020tight}
Yossi Arjevani, Ohad Shamir, and Nathan Srebro.
\newblock A tight convergence analysis for stochastic gradient descent with
  delayed updates.
\newblock In \emph{Proceedings of the 31st International Conference on
  Algorithmic Learning Theory}, pages 111--132. PMLR, 2020.

\bibitem[Arjevani et~al.(2023)Arjevani, Carmon, Duchi, Foster, Srebro, and
  Woodworth]{arjevani2023lower}
Yossi Arjevani, Yair Carmon, John~C Duchi, Dylan~J Foster, Nathan Srebro, and
  Blake Woodworth.
\newblock Lower bounds for non-convex stochastic optimization.
\newblock \emph{Mathematical Programming}, 199\penalty0 (1):\penalty0 165--214,
  2023.

\bibitem[Assran et~al.(2020)Assran, Aytekin, Feyzmahdavian, Johansson, and
  Rabbat]{assran2020advances}
Mahmoud Assran, Arda Aytekin, Hamid~Reza Feyzmahdavian, Mikael Johansson, and
  Michael~G Rabbat.
\newblock Advances in asynchronous parallel and distributed optimization.
\newblock \emph{Proceedings of the IEEE}, 108\penalty0 (11):\penalty0
  2013--2031, 2020.

\bibitem[Avdiukhin and Kasiviswanathan(2021)]{avdiukhin2021federated}
Dmitrii Avdiukhin and Shiva Kasiviswanathan.
\newblock Federated learning under arbitrary communication patterns.
\newblock In \emph{Proceedings of the 38th International Conference on Machine
  Learning}, pages 425--435. PMLR, 2021.

\bibitem[Blatt et~al.(2007)Blatt, Hero, and Gauchman]{blatt2007convergent}
Doron Blatt, Alfred~O Hero, and Hillel Gauchman.
\newblock A convergent incremental gradient method with a constant step size.
\newblock \emph{SIAM Journal on Optimization}, 18\penalty0 (1):\penalty0
  29--51, 2007.

\bibitem[Chen et~al.(2016)Chen, Pan, Monga, Bengio, and
  Jozefowicz]{chen2016revisiting}
Jianmin Chen, Xinghao Pan, Rajat Monga, Samy Bengio, and Rafal Jozefowicz.
\newblock Revisiting distributed synchronous {SGD}.
\newblock \emph{arXiv preprint arXiv:1604.00981}, 2016.

\bibitem[Chen et~al.(2020)Chen, Ning, Slawski, and
  Rangwala]{chen2020asynchronous}
Yujing Chen, Yue Ning, Martin Slawski, and Huzefa Rangwala.
\newblock Asynchronous online federated learning for edge devices with
  non-{IID} data.
\newblock In \emph{Proceedings of the 2020 IEEE International Conference on Big
  Data}, pages 15--24. IEEE, 2020.

\bibitem[Cotter et~al.(2011)Cotter, Shamir, Srebro, and
  Sridharan]{cotter2011better}
Andrew Cotter, Ohad Shamir, Nati Srebro, and Karthik Sridharan.
\newblock Better mini-batch algorithms via accelerated gradient methods.
\newblock In \emph{Advances in Neural Information Processing Systems 24}, 2011.

\bibitem[De~Sa et~al.(2015)De~Sa, Zhang, Olukotun, and R{\'e}]{de2015taming}
Christopher~M De~Sa, Ce~Zhang, Kunle Olukotun, and Christopher R{\'e}.
\newblock Taming the wild: A unified analysis of {Hogwild}-style algorithms.
\newblock In \emph{Advances in Neural Information Processing Systems 28}, 2015.

\bibitem[Defazio et~al.(2014)Defazio, Bach, and
  Lacoste-Julien]{defazio2014saga}
Aaron Defazio, Francis Bach, and Simon Lacoste-Julien.
\newblock {SAGA}: A fast incremental gradient method with support for
  non-strongly convex composite objectives.
\newblock In \emph{Advances in Neural Information Processing Systems 27}, 2014.

\bibitem[Dekel et~al.(2012)Dekel, Gilad-Bachrach, Shamir, and
  Xiao]{dekel2012optimal}
Ofer Dekel, Ran Gilad-Bachrach, Ohad Shamir, and Lin Xiao.
\newblock Optimal distributed online prediction using mini-batches.
\newblock \emph{Journal of Machine Learning Research}, 13\penalty0 (1), 2012.

\bibitem[Dutta et~al.(2021)Dutta, Wang, and Joshi]{dutta2021slow}
Sanghamitra Dutta, Jianyu Wang, and Gauri Joshi.
\newblock Slow and stale gradients can win the race.
\newblock \emph{IEEE Journal on Selected Areas in Information Theory},
  2\penalty0 (3):\penalty0 1012--1024, 2021.

\bibitem[Even et~al.(2024)Even, Koloskova, and
  Massouli{\'e}]{even2024asynchronous}
Mathieu Even, Anastasia Koloskova, and Laurent Massouli{\'e}.
\newblock Asynchronous {SGD} on graphs: A unified framework for asynchronous
  decentralized and federated optimization.
\newblock In \emph{Proceedings of the 27th International Conference on
  Artificial Intelligence and Statistics}, pages 64--72. PMLR, 2024.

\bibitem[Feyzmahdavian et~al.(2016)Feyzmahdavian, Aytekin, and
  Johansson]{feyzmahdavian2016asynchronous}
Hamid~Reza Feyzmahdavian, Arda Aytekin, and Mikael Johansson.
\newblock An asynchronous mini-batch algorithm for regularized stochastic
  optimization.
\newblock \emph{IEEE Transactions on Automatic Control}, 61\penalty0
  (12):\penalty0 3740--3754, 2016.

\bibitem[Fraboni et~al.(2023)Fraboni, Vidal, Kameni, and
  Lorenzi]{fraboni2023general}
Yann Fraboni, Richard Vidal, Laetitia Kameni, and Marco Lorenzi.
\newblock A general theory for federated optimization with asynchronous and
  heterogeneous clients updates.
\newblock \emph{Journal of Machine Learning Research}, 24\penalty0
  (110):\penalty0 1--43, 2023.

\bibitem[Gao et~al.(2021)Gao, Wu, and Rossi]{gao2021provable}
Hongchang Gao, Gang Wu, and Ryan Rossi.
\newblock Provable distributed stochastic gradient descent with delayed
  updates.
\newblock In \emph{Proceedings of the 2021 SIAM International Conference on
  Data Mining}, pages 441--449. SIAM, 2021.

\bibitem[Glasgow and Wootters(2022)]{glasgow2022asynchronous}
Margalit~R Glasgow and Mary Wootters.
\newblock Asynchronous distributed optimization with stochastic delays.
\newblock In \emph{Proceedings of the 25th International Conference on
  Artificial Intelligence and Statistics}, pages 9247--9279. PMLR, 2022.

\bibitem[Goyal et~al.(2017)Goyal, Doll{\'a}r, Girshick, Noordhuis, Wesolowski,
  Kyrola, Tulloch, Jia, and He]{goyal2017accurate}
Priya Goyal, Piotr Doll{\'a}r, Ross Girshick, Pieter Noordhuis, Lukasz
  Wesolowski, Aapo Kyrola, Andrew Tulloch, Yangqing Jia, and Kaiming He.
\newblock Accurate, large minibatch {SGD}: Training {ImageNet} in 1 hour.
\newblock \emph{arXiv preprint arXiv:1706.02677}, 2017.

\bibitem[Gu et~al.(2021)Gu, Huang, Zhang, and Huang]{gu2021fast}
Xinran Gu, Kaixuan Huang, Jingzhao Zhang, and Longbo Huang.
\newblock Fast federated learning in the presence of arbitrary device
  unavailability.
\newblock In \emph{Advances in Neural Information Processing Systems 34}, pages
  12052--12064, 2021.

\bibitem[Gurbuzbalaban et~al.(2017)Gurbuzbalaban, Ozdaglar, and
  Parrilo]{gurbuzbalaban2017convergence}
Mert Gurbuzbalaban, Asuman Ozdaglar, and Pablo~A Parrilo.
\newblock On the convergence rate of incremental aggregated gradient
  algorithms.
\newblock \emph{SIAM Journal on Optimization}, 27\penalty0 (2):\penalty0
  1035--1048, 2017.

\bibitem[Hsu et~al.(2019)Hsu, Qi, and Brown]{hsu2019measuring}
Tzu-Ming~Harry Hsu, Hang Qi, and Matthew Brown.
\newblock Measuring the effects of non-identical data distribution for
  federated visual classification.
\newblock \emph{arXiv preprint arXiv:1909.06335}, 2019.

\bibitem[Islamov et~al.(2024)Islamov, Safaryan, and
  Alistarh]{islamov2024asgrad}
Rustem Islamov, Mher Safaryan, and Dan Alistarh.
\newblock {AsGrad}: A sharp unified analysis of asynchronous-{SGD} algorithms.
\newblock In \emph{Proceedings of the 27th International Conference on
  Artificial Intelligence and Statistics}, pages 649--657. PMLR, 2024.

\bibitem[Kairouz et~al.(2021)Kairouz, McMahan, Avent, Bellet, Bennis, Bhagoji,
  Bonawitz, Charles, Cormode, Cummings, et~al.]{kairouz2021advances}
Peter Kairouz, H~Brendan McMahan, Brendan Avent, Aur{\'e}lien Bellet, Mehdi
  Bennis, Arjun~Nitin Bhagoji, Kallista Bonawitz, Zachary Charles, Graham
  Cormode, Rachel Cummings, et~al.
\newblock Advances and open problems in federated learning.
\newblock \emph{Foundations and Trends{\textregistered} in Machine Learning},
  14\penalty0 (1--2):\penalty0 1--210, 2021.

\bibitem[Karimireddy et~al.(2020)Karimireddy, Kale, Mohri, Reddi, Stich, and
  Suresh]{karimireddy2020scaffold}
Sai~Praneeth Karimireddy, Satyen Kale, Mehryar Mohri, Sashank Reddi, Sebastian
  Stich, and Ananda~Theertha Suresh.
\newblock Scaffold: Stochastic controlled averaging for federated learning.
\newblock In \emph{Proceedings of the 37th International Conference on Machine
  Learning}, pages 5132--5143. PMLR, 2020.

\bibitem[Khaled and Richt{\'a}rik(2023)]{khaled2023better}
Ahmed Khaled and Peter Richt{\'a}rik.
\newblock Better theory for {SGD} in the nonconvex world.
\newblock \emph{Transactions on Machine Learning Research}, 2023.

\bibitem[Koloskova et~al.(2022)Koloskova, Stich, and
  Jaggi]{koloskova2022sharper}
Anastasiia Koloskova, Sebastian~U Stich, and Martin Jaggi.
\newblock Sharper convergence guarantees for asynchronous {SGD} for distributed
  and federated learning.
\newblock In \emph{Advances in Neural Information Processing Systems 35}, pages
  17202--17215, 2022.

\bibitem[Krizhevsky et~al.(2009)Krizhevsky, Hinton,
  et~al.]{krizhevsky2009learning}
Alex Krizhevsky, Geoffrey Hinton, et~al.
\newblock Learning multiple layers of features from tiny images.
\newblock Technical report, University of Toronto, 2009.
\newblock Available online:
  \url{https://www.cs.utoronto.ca/~kriz/learning-features-2009-TR.pdf}.

\bibitem[Leblond et~al.(2018)Leblond, Pedregosa, and
  Lacoste-Julien]{leblond2018improved}
R{\'e}mi Leblond, Fabian Pedregosa, and Simon Lacoste-Julien.
\newblock Improved asynchronous parallel optimization analysis for stochastic
  incremental methods.
\newblock \emph{Journal of Machine Learning Research}, 19\penalty0
  (81):\penalty0 1--68, 2018.

\bibitem[Leconte et~al.(2024{\natexlab{a}})Leconte, Jonckheere, Samsonov, and
  Moulines]{leconte2024queuing}
Louis Leconte, Matthieu Jonckheere, Sergey Samsonov, and Eric Moulines.
\newblock Queuing dynamics of asynchronous federated learning.
\newblock In \emph{Proceedings of the 27th International Conference on
  Artificial Intelligence and Statistics}, pages 1711--1719. PMLR,
  2024{\natexlab{a}}.

\bibitem[Leconte et~al.(2024{\natexlab{b}})Leconte, Moulines,
  et~al.]{leconte2024favano}
Louis Leconte, Eric Moulines, et~al.
\newblock {FAVANO}: Federated averaging with asynchronous nodes.
\newblock In \emph{Proceedings of the 2024 IEEE International Conference on
  Acoustics, Speech and Signal Processing}, pages 5665--5669. IEEE,
  2024{\natexlab{b}}.

\bibitem[Li et~al.(2022)Li, Diao, Chen, and He]{li2022federated}
Qinbin Li, Yiqun Diao, Quan Chen, and Bingsheng He.
\newblock Federated learning on non-{IID} data silos: An experimental study.
\newblock In \emph{Proceedings of the 2022 IEEE 38th International Conference
  on Data Engineering}, pages 965--978. IEEE, 2022.

\bibitem[Li et~al.(2020)Li, Sahu, Talwalkar, and Smith]{li2020federated}
Tian Li, Anit~Kumar Sahu, Ameet Talwalkar, and Virginia Smith.
\newblock Federated learning: Challenges, methods, and future directions.
\newblock \emph{IEEE Signal Processing Magazine}, 37\penalty0 (3):\penalty0
  50--60, 2020.

\bibitem[Lian et~al.(2015)Lian, Huang, Li, and Liu]{lian2015asynchronous}
Xiangru Lian, Yijun Huang, Yuncheng Li, and Ji~Liu.
\newblock Asynchronous parallel stochastic gradient for nonconvex optimization.
\newblock In \emph{Advances in Neural Information Processing Systems 28}, 2015.

\bibitem[Lian et~al.(2018)Lian, Zhang, Zhang, and Liu]{lian2018asynchronous}
Xiangru Lian, Wei Zhang, Ce~Zhang, and Ji~Liu.
\newblock Asynchronous decentralized parallel stochastic gradient descent.
\newblock In \emph{Proceedings of the 37th International Conference on Machine
  Learning}, pages 3043--3052. PMLR, 2018.

\bibitem[Mania et~al.(2017)Mania, Pan, Papailiopoulos, Recht, Ramchandran, and
  Jordan]{mania2017perturbed}
Horia Mania, Xinghao Pan, Dimitris Papailiopoulos, Benjamin Recht, Kannan
  Ramchandran, and Michael~I Jordan.
\newblock Perturbed iterate analysis for asynchronous stochastic optimization.
\newblock \emph{SIAM Journal on Optimization}, 27\penalty0 (4):\penalty0
  2202--2229, 2017.

\bibitem[Mishchenko et~al.(2022)Mishchenko, Bach, Even, and
  Woodworth]{mishchenko2022asynchronous}
Konstantin Mishchenko, Francis Bach, Mathieu Even, and Blake~E Woodworth.
\newblock Asynchronous {SGD} beats minibatch {SGD} under arbitrary delays.
\newblock In \emph{Advances in Neural Information Processing Systems 35}, pages
  420--433, 2022.

\bibitem[Nedi{\'c} et~al.(2001)Nedi{\'c}, Bertsekas, and
  Borkar]{nedic2001distributed}
Angelia Nedi{\'c}, Dimitri~P Bertsekas, and Vivek~S Borkar.
\newblock Distributed asynchronous incremental subgradient methods.
\newblock \emph{Studies in Computational Mathematics}, 8\penalty0 (C):\penalty0
  381--407, 2001.

\bibitem[Nguyen et~al.(2022)Nguyen, Malik, Zhan, Yousefpour, Rabbat, Malek, and
  Huba]{nguyen2022federated}
John Nguyen, Kshitiz Malik, Hongyuan Zhan, Ashkan Yousefpour, Mike Rabbat, Mani
  Malek, and Dzmitry Huba.
\newblock Federated learning with buffered asynchronous aggregation.
\newblock In \emph{Proceedings of the 25th International Conference on
  Artificial Intelligence and Statistics}, pages 3581--3607. PMLR, 2022.

\bibitem[Recht et~al.(2011)Recht, Re, Wright, and Niu]{recht2011hogwild}
Benjamin Recht, Christopher Re, Stephen Wright, and Feng Niu.
\newblock {Hogwild}!: A lock-free approach to parallelizing stochastic gradient
  descent.
\newblock In \emph{Advances in Neural Information Processing Systems 24}, 2011.

\bibitem[Roux et~al.(2012)Roux, Schmidt, and Bach]{roux2012stochastic}
Nicolas Roux, Mark Schmidt, and Francis Bach.
\newblock A stochastic gradient method with an exponential convergence rate for
  finite training sets.
\newblock In \emph{Advances in Neural Information Processing Systems 25}, 2012.

\bibitem[Schmidt et~al.(2017)Schmidt, Le~Roux, and Bach]{schmidt2017minimizing}
Mark Schmidt, Nicolas Le~Roux, and Francis Bach.
\newblock Minimizing finite sums with the stochastic average gradient.
\newblock \emph{Mathematical Programming}, 162:\penalty0 83--112, 2017.

\bibitem[Stich and Karimireddy(2020)]{stich2020error}
Sebastian~U Stich and Sai~Praneeth Karimireddy.
\newblock The error-feedback framework: {SGD} with delayed gradients.
\newblock \emph{Journal of Machine Learning Research}, 21\penalty0
  (237):\penalty0 1--36, 2020.

\bibitem[Sun et~al.(2023)Sun, Mao, and Zhang]{sun2023mimic}
Yuchang Sun, Yuyi Mao, and Jun Zhang.
\newblock {MimiC}: Combating client dropouts in federated learning by mimicking
  central updates.
\newblock \emph{IEEE Transactions on Mobile Computing}, 2023.

\bibitem[Toghani and Uribe(2022)]{toghani2022unbounded}
Mohammad~Taha Toghani and C{\'e}sar~A Uribe.
\newblock Unbounded gradients in federated learning with buffered asynchronous
  aggregation.
\newblock In \emph{Proceedings of the 2022 58th Annual Allerton Conference on
  Communication, Control, and Computing}, pages 1--8. IEEE, 2022.

\bibitem[Vanli et~al.(2018)Vanli, Gurbuzbalaban, and Ozdaglar]{vanli2018global}
N~Denizcan Vanli, Mert Gurbuzbalaban, and Asuman Ozdaglar.
\newblock Global convergence rate of proximal incremental aggregated gradient
  methods.
\newblock \emph{SIAM Journal on Optimization}, 28\penalty0 (2):\penalty0
  1282--1300, 2018.

\bibitem[Verbraeken et~al.(2020)Verbraeken, Wolting, Katzy, Kloppenburg,
  Verbelen, and Rellermeyer]{verbraeken2020survey}
Joost Verbraeken, Matthijs Wolting, Jonathan Katzy, Jeroen Kloppenburg, Tim
  Verbelen, and Jan~S Rellermeyer.
\newblock A survey on distributed machine learning.
\newblock \emph{ACM Computing Surveys}, 53\penalty0 (2):\penalty0 1--33, 2020.

\bibitem[Wang et~al.(2023{\natexlab{a}})Wang, Jin, Wai, and Gu]{wang2023linear}
Xiaolu Wang, Cheng Jin, Hoi-To Wai, and Yuantao Gu.
\newblock Linear speedup of incremental aggregated gradient methods on
  streaming data.
\newblock In \emph{Proceedings of the 2023 62nd IEEE Conference on Decision and
  Control}, pages 4314--4319. IEEE, 2023{\natexlab{a}}.

\bibitem[Wang et~al.(2024)Wang, Li, Jin, and Zhang]{wang2024achieving}
Xiaolu Wang, Zijian Li, Shi Jin, and Jun Zhang.
\newblock Achieving linear speedup in asynchronous federated learning with
  heterogeneous clients.
\newblock \emph{arXiv preprint arXiv:2402.11198}, 2024.

\bibitem[Wang et~al.(2023{\natexlab{b}})Wang, Cao, Wu, Chen, and
  Chen]{wang2023tackling}
Yujia Wang, Yuanpu Cao, Jingcheng Wu, Ruoyu Chen, and Jinghui Chen.
\newblock Tackling the data heterogeneity in asynchronous federated learning
  with cached update calibration.
\newblock In \emph{Proceedings of the 12th International Conference on Learning
  Representations}, 2023{\natexlab{b}}.

\bibitem[Xie et~al.(2019)Xie, Koyejo, and Gupta]{xie2019asynchronous}
Cong Xie, Sanmi Koyejo, and Indranil Gupta.
\newblock Asynchronous federated optimization.
\newblock \emph{arXiv preprint arXiv:1903.03934}, 2019.

\bibitem[Yang et~al.(2019)Yang, Liu, Chen, and Tong]{yang2019federated}
Qiang Yang, Yang Liu, Tianjian Chen, and Yongxin Tong.
\newblock Federated machine learning: Concept and applications.
\newblock \emph{ACM Transactions on Intelligent Systems and Technology},
  10\penalty0 (2):\penalty0 1--19, 2019.

\bibitem[Yurochkin et~al.(2019)Yurochkin, Agarwal, Ghosh, Greenewald, Hoang,
  and Khazaeni]{yurochkin2019bayesian}
Mikhail Yurochkin, Mayank Agarwal, Soumya Ghosh, Kristjan Greenewald, Nghia
  Hoang, and Yasaman Khazaeni.
\newblock Bayesian nonparametric federated learning of neural networks.
\newblock In \emph{Proceedings of the 36th International Conference on Machine
  Learning}, pages 7252--7261. PMLR, 2019.

\bibitem[Zakerinia et~al.(2022)Zakerinia, Talaei, Nadiradze, and
  Alistarh]{zakerinia2022quafl}
Hossein Zakerinia, Shayan Talaei, Giorgi Nadiradze, and Dan Alistarh.
\newblock Qu{AFL}: Federated averaging can be both asynchronous and
  communication-efficient.
\newblock \emph{arXiv preprint arXiv:2206.10032}, 2022.

\bibitem[Zhang et~al.(2023)Zhang, Liu, Lin, Wu, Zhou, Jiang, and
  Ji]{zhang2023no}
Feilong Zhang, Xianming Liu, Shiyi Lin, Gang Wu, Xiong Zhou, Junjun Jiang, and
  Xiangyang Ji.
\newblock No one idles: Efficient heterogeneous federated learning with
  parallel edge and server computation.
\newblock In \emph{Proceedings of the 40th International Conference on Machine
  Learning}, pages 41399--41413. PMLR, 2023.

\end{thebibliography}
	
	\appendix
	\section{Additional Related Works}\label{sec:related}
	
	\textbf{Other Variants of ASGD.} 
	While numerous variants of ASGD have been developed, they mostly address simpler scenarios with homogeneous data \citep{agarwal2011distributed,lian2015asynchronous,feyzmahdavian2016asynchronous,leblond2018improved,stich2020error,arjevani2020tight,dutta2021slow}, where all workers operate on the same loss function and possess data with an identical probability distribution. In this setup, Problem \eqref{eq:opt} reduces to
	\[
	\min_{\bw \in \mathbb{R}^p} ~ \mathbb{E}_{\bxi \sim \mathbb{P}_1} \left[ f_1 (\bw; \bxi) \right].
	\]
	Another line of research explores ASGD for homogeneous data under the \textit{model parallelism} setting \citep{recht2011hogwild,de2015taming,mania2017perturbed}, where each worker is solely responsible for updating a specific block of the model parameters independently, such as a distinct layer of a neural network.
	The assumption of data homogeneity is valid for \textit{shared-memory architectures}; however, this assumption becomes highly idealistic in data-parallelism scenarios, where workers may hold significantly diverse local datasets, especially in applications such as Internet of Things, healthcare, and financial services \citep{li2020federated,kairouz2021advances}.
	An independent line of works have also considered ASGD in decentralized networks \citep{lian2018asynchronous,even2024asynchronous}, contrasting with the more commonly studied centralized architectures, as depicted in Figure \ref{fig:asgd-illustration}, that rely on a central server.
	
	\textbf{Asynchronous Federated Learning.}
	Federated learning (FL) is an emerging distributed machine learning paradigm that pays particular attention to data privacy and heterogeneity. 
	Asynchronous federated learning algorithms \citep{xie2019asynchronous,chen2020asynchronous,nguyen2022federated,zakerinia2022quafl,wang2023tackling,fraboni2023general,wang2024achieving,leconte2024favano} share similarities with ASGD but typically include a local update strategy that potentially reduces the communication frequency between the server and workers. 
	Among these works, FedBuff \citep{nguyen2022federated} serves as a representative algorithm where workers operate independently, and the server waits for a subset of workers $\mathcal{C}_t$ to submit their local updates in each global iteration. The local and global updates of FedBuff proceeds as follows:
	\begin{align*}
		&\bw_i^{\tau_i(t),k} = \bw_i^{\tau_i(t),k-1} - \eta_{\ell} \nabla f_i ( \bw_i^{\tau_i(t),k-1}; \bxi_i^{t,k} ), ~k = 1, 2, \dots, K,~i \in \mathcal{C}_t,
		\\
		&\bw^{t} = \bw^{t-1} - \frac{\eta_g}{|\mathcal{C}_t|} \sum_{i \in \mathcal{C}_t} (\bw_i^{\tau_i(t),0}- \bw_i^{\tau_i(t),K}), ~t = 1, 2, \dots,
	\end{align*}
	where $\eta_{\ell}$ and $\eta_g$ are the local and global step sizes, respectively.
	This approach can be regarded as a \textit{semi-asynchronous} algorithm designed to decrease communication frequency at the expense of increased waiting time.
	Nevertheless, the local update strategy in federated learning leads to the \textit{client drift} phenomenon \citep{karimireddy2020scaffold,sun2023mimic}, where local models at each worker tend to deviate from the global model. Therefore, asynchronous federated learning algorithms typically require either data heterogeneity or function dissimilarity to be bounded, so that local models remain closely aligned with the global model throughout the training process.
	
	\textbf{Incremental Aggregated Gradient (IAG)-Type Methods:} 
	The algorithmic concept of DuDe-ASGD is rooted in the well-established IAG methods \citep{blatt2007convergent,gurbuzbalaban2017convergence,vanli2018global}, which update the model parameters by using a combination of new and previously computed gradients.
	When solving Problem \eqref{eq:opt}, the iterative formula of IAG can be expressed as:
	\[
	\bw^t = \bw^{t-1} - \frac{\eta_t}{n} \sum_{i=1}^{n} \nabla F_i ( \bw^{t-\tau_i(t)}), ~t = 1, 2, \dots,
	\]
	where $\{\eta_t\}_{t \geq 1}$ are step sizes.
	IAG is suitable for asynchronous distributed implementations.
	SAG \citep{roux2012stochastic,schmidt2017minimizing} and
	SAGA \citep{defazio2014saga} are randomized versions of IAG, where the index $i$ of the component function to be updated is selected at random in each iteration.
	\cite{glasgow2022asynchronous} developed ADSAGA, an extension of SAGA to the asynchronous setting, assuming a stochastic delay model and that the server is aware of the delay distribution.
	Nevertheless, these algorithms all assume that the \textit{exact gradients} $\nabla F_i(\cdot)$ can be evaluated by each worker. This is a fundamental difference from our DuDe-ASGD, which relies solely on \textit{stochastic gradients} $\nabla f_i(\cdot,\bxi_i)$ for any $\bxi_i \in \Xi_i$.	
	
	\section{Proofs of Main Results}\label{sec:appendix-proofs}
	
	For random variables $P,Q$ and function $h$, we denote by
	\[
	\E_P[h(P,Q)] \coloneqq \E[h(P,Q) \mid Q]
	\]
	the \textit{conditional expectation} with respect to $P$ while holding $Q$ constant.
	
	\subsection{Technical Lemmas}
	
	\begin{lemma}\label{lem:variance}
		Suppose that Assumptions \ref{as:unbiased} and \ref{as:sigma} hold. Then, it holds for all $t \geq 1$ that
		\begin{align}
			\E \left\| \frac{1}{n} \sum_{i=1}^{n} \left( \nabla f_i(\bw^{t-\tau_i(t)}; \bxi_i^{t-d_i(t)}) - \nabla F_i(\bw^{t-\tau_i(t)}) \right) \right\|_2^2 \leq \frac{\sigma^2}{n}.
			\nonumber
		\end{align}
	\end{lemma}
	\begin{proof}
		Expanding the squared norm gives
		\begin{align}
			&~\E \left\| \frac{1}{n} \sum_{i=1}^{n} \left( \nabla f_i( \bw^{t-\tau_i(t)}; \bxi_i^{t-d_i(t)} ) - \nabla F_i( \bw^{t-\tau_i(t)} ) \right)   \right\|_2^2
			\nonumber
			\\
			=&~ \frac{1}{n^2} \sum_{i=1}^{n} \E \| \nabla f_i( \bw^{t-\tau_i(t)}; \bxi_i^{t-d_i(t)} ) - \nabla F_i(\bw^{t-\tau_i(t)}) \|_2^2
			\nonumber
			\\
			&~+ \frac{1}{n^2} \sum_{i \neq j} \underbrace{\E \left\langle \nabla f_i( \bw^{t-\tau_i(t)}; \bxi_i^{t-d_i(t)} ) - \nabla F_i( \bw^{t-\tau_i(t)} ),  \nabla f_j(\bw^{t-\tau_j(t)}; \bxi_j^{t-d_j(t)}) - \nabla F_j(\bw^{t-\tau_j(t)}) \right\rangle}_{Y_{ij}}
			\label{eq:Yij}
		\end{align}
		To simplify $Y_{ij}$ for $i,j \in [n]$ s.t. $i \neq j$, we assume without loss of generality that $d_i(t) \geq d_j(t)$.
		Then, $t-\tau_i(t) \leq t-d_j(t)$ and thus $\bxi_j^{t-d_j(t)}$ is independent of $\bw^{t-\tau_i(t)}$.
		Hence, using the law of total expectation and Assumption \ref{as:unbiased}, we have
		\begin{align}
			Y_{ij}=&~\E \left\langle \nabla f_i( \bw^{t-\tau_i(t)}; \bxi_i^{t-d_i(t)} ) - \nabla F_i( \bw^{t-\tau_i(t)} ),  \nabla f_j( \bw^{t-\tau_j(t)}; \bxi_j^{t-d_j(t)} ) - \nabla F_j( \bw^{t-\tau_j(t)} ) \right\rangle 
			\nonumber
			\\
			=&~ \E \left[ \E_{\bxi_j^{t-d_j(t)} \sim \mathbb{P}_j} \left\langle \nabla f_i( \bw^{t-\tau_i(t)}; \bxi_i^{t-d_i(t)} ) - \nabla F_i( \bw^{t-\tau_i(t)} ),  \nabla f_j( \bw^{t-\tau_j(t)}; \bxi_j^{t-d_j(t)} ) - \nabla F_j( \bw^{t-\tau_j(t)} ) \right\rangle \right]
			\nonumber
			\\
			=&~ \E \left[ \left\langle \nabla f_i( \bw^{t-\tau_i(t)}; \bxi_i^{t-d_i(t)} ) - \nabla F_i( \bw^{t-\tau_i(t)} ),  \E_{\bxi_j^{t-d_j(t)}} \left[ \nabla f_j( \bw^{t-\tau_j(t)}; \bxi_j^{t-d_j(t)} ) - \nabla F_j( \bw^{t-\tau_j(t)} ) \right] \right\rangle \right]
			\nonumber
			\\
			=&~ {0}.
			\nonumber
		\end{align}
		Substituting this back into \eqref{eq:Yij} gives 
		\begin{align}
			&~\E \left\| \frac{1}{n} \sum_{i=1}^{n} \left( \nabla f_i( \bw^{t-\tau_i(t)}; \bxi_i^{t-d_i(t)} ) - \nabla F_i( \bw^{t-\tau_i(t)} ) \right)   \right\|_2^2
			\nonumber
			\\
			=&~ \frac{1}{n^2} \sum_{i=1}^{n} \E \| \nabla f_i( \bw^{t-\tau_i(t)}; \bxi_i^{t-d_i(t)} ) - \nabla F_i(\bw^{t-\tau_i(t)}) \|_2^2
			\nonumber
			\\
			=&~ \frac{1}{n^2} \sum_{i=1}^{n} \E \left[ \E \left[ \| \nabla f_i( \bw^{t-\tau_i(t)}; \bxi_i^{t-d_i(t)} ) - \nabla F_i(\bw^{t-\tau_i(t)}) \|_2^2 \mv \bw^{t-\tau_i(t)} \right] \right]
			\nonumber
			\\
			\leq&~ \frac{\sigma^2}{n}.
			\label{eq:X1}
		\end{align}
		where the second equality holds due to the law of total expectation and the inequality follows from Assumption \ref{as:sigma}.
	\end{proof}
	
	\begin{lemma}\label{lem:diff}
		Suppose that Assumptions \ref{as:unbiased} and \ref{as:sigma} hold. Then, it holds for all $i \in [n]$ and $t \geq 1$ that
		\begin{align}
			\E \| \bw^t - \bw^{t-\tau_i(t)} \|_2^2
			\leq 2 \tau_{\max}^2 \eta^2 \frac{\sigma^2}{n} + 2 \tau_{\max} \eta^2 \sum_{s=1+[t-\tau_{\max}]_+}^{t} \E \left\| \frac{1}{n} \sum_{j=1}^{n} \nabla F_j(\bw^{s-\tau_j(s)}) \right\|_2^2.
			\nonumber
		\end{align}
	\end{lemma}
	\begin{proof}
		For all $i \in [n]$ and $t \geq 1$, it follows from the telescoping sum 
		\[
		\sum_{s = 1+t-\tau_i(t)}^{t} \left( \bw^s - \bw^{s-1} \right) = \bw^t - \bw^{t-\tau_i(t)}
		\]
		and the iterative formula \eqref{eq:iter} that
		\begin{align}
			&\E \| \bw^t - \bw^{t-\tau_i(t)} \|_2^2
			\nonumber
			\\
			=~&  \E \left\| \sum_{s = 1+t-\tau_i(t)}^{t} \left( \bw^{s} - \bw^{s-1} \right) \right\|_2^2
			\nonumber
			\\ 
			=~& \E \left\| \sum_{s = 1+t-\tau_i(t)}^{t} \eta \bg^s \right\|_2^2
			\nonumber
			\\
			=~& \E \left\| \sum_{s = 1+t-\tau_i(t)}^{t} \frac{\eta}{n} \sum_{j=1}^{n} \nabla f_j(\bw^{s-\tau_j(s)}; \bxi_j^{s-d_j(s)}) \right\|_2^2
			\nonumber
			\\
			=~& \frac{\eta^2}{n^2} \E \left\| \sum_{s = 1+t-\tau_i(t)}^{t} \sum_{j=1}^{n} \left( \nabla f_j(\bw^{s-\tau_j(s)}; \bxi_j^{s-d_j(s)}) - \nabla F_j(\bw^{s-\tau_j(s)})
			+ \nabla F_j(\bw^{s-\tau_j(s)}) \right) \right\|_2^2
			\nonumber
			\\
			\leq~& \frac{2 \eta^2}{n^2} \underbrace{\E \left\| \sum_{s = 1+t-\tau_i(t)}^{t} \sum_{j=1}^{n} \left( \nabla f_j(\bw^{s-\tau_j(s)}; \bxi_j^{s-d_j(s)}) - \nabla F_j(\bw^{s-\tau_j(s)}) \right) \right\|_2^2}_{\Phi_1}
			\nonumber
			\\
			&+ \frac{2 \eta^2}{n^2} \underbrace{\E \left\| \sum_{s = 1+t-\tau_i(t)}^{t} \sum_{j=1}^{n} \nabla F_j(\bw^{s-\tau_j(s)}) \right\|_2^2}_{\Phi_2},
			\label{eq:diff-tmp}
		\end{align}
		where the inequality uses the fact that $\| \bx + \by \|_2^2 \leq 2 \| \bx \|_2^2 + 2 \| \by \|_2^2$ for vectors $\bx$ and $\by$. 
		Subsequently, we upper bound $\Phi_1$ and $\Phi_2$, respectively.
		
		\underline{Upper bounding $Y_1$:} 
		Expanding $\Phi_1$, we have
		\begin{align}
			&\Phi_1 = \sum_{s = 1+t-\tau_i(t)}^{t} \E \left\| \sum_{j=1}^{n} \left( \nabla f_j(\bw^{s-\tau_j(s)}; \bxi_i^{s-d_j(s)}) - \nabla F_j(\bw^{s-\tau_j(s)}) \right) \right\|_2^2
			\nonumber
			\\
			& + \hspace{-3mm} \sum_{\substack{s,s': s \neq s', \\ 1+t-\tau_i(t) \leq s,s' \leq t}} 
			\underbrace{\E \hspace{-0.5mm}\left\langle\hspace{-0.2mm} \sum_{j=1}^{n} \hspace{-1mm}\left(\hspace{-1mm} \nabla f_j(\bw^{s-\tau_j(s)}; \bxi_j^{s-d_j(s)}) \hspace{-0.6mm}-\hspace{-1mm} \nabla F_j(\bw^{s-\tau_j(s)}) \hspace{-0.8mm}\right)\hspace{-1mm}, 
				\hspace{-0.5mm}\sum_{j=1}^{n}\hspace{-1mm} \left(\hspace{-0.5mm} \nabla f_j(\bw^{s'\hspace{-0.5mm}-\hspace{-0.2mm}\tau_j(s')}; \bxi_j^{s'\hspace{-0.5mm}-\hspace{-0.2mm}d_j(s')}) \hspace{-0.8mm}-\hspace{-1mm} \nabla F_j(\bw^{s'\hspace{-0.5mm}-\hspace{-0.2mm}\tau_j(s')}) \hspace{-0.5mm}\right)\hspace{-1.5mm} \right\rangle\hspace{-0.5mm}}_{Z^{ss'}}\hspace{-0.5mm}.
			\label{eq:Y1}
		\end{align}
		Inspecting the inner product terms in \eqref{eq:Y1}, we note that for all $s,s' \in [1+t-\tau_i(t), t]$ s.t. $s \neq s'$,
		\begin{align}
			Z^{ss'} 
			=&~ \E \left\langle \sum_{j=1}^{n} \left( \nabla f_j(\bw^{s-\tau_j(s)}; \bxi_j^{s-d_j(s)}) - \nabla F_j(\bw^{s-\tau_j(s)}) \right), \sum_{j=1}^{n} \left( \nabla f_j(\bw^{s'-\tau_j(s')}; \bxi_j^{s'-d_j(s')}) - \nabla F_j(\bw^{s'-\tau_j(s')}) \right) \right\rangle
			\nonumber
			\\			
			=&~ \E \left[ \sum_{i=1}^{n} \sum_{j=1}^{n} \left\langle \nabla f_i(\bw^{s-\tau_i(s)}; \bxi_i^{s-d_i(s)}) - \nabla F_i(\bw^{s-\tau_i(s)}) , \nabla f_j(\bw^{s'-\tau_j(s')}; \bxi_j^{s'-d_j(s')}) - \nabla F_j(\bw^{s'-\tau_j(s')}) \right\rangle \right]
			\nonumber
			\\
			=&~ \sum_{i=1}^{n} \sum_{j=1}^{n} \E \left\langle \nabla f_i(\bw^{s-\tau_i(s)}; \bxi_i^{s-d_i(s)}) - \nabla F_i(\bw^{s-\tau_i(s)}) , \nabla f_j(\bw^{s'-\tau_j(s')}; \bxi_j^{s'-d_j(s')}) - \nabla F_j(\bw^{s'-\tau_j(s')}) \right\rangle
			\nonumber
			\\
			=&~ \sum_{j=1}^{n} \E \left\langle \nabla f_j(\bw^{s-\tau_j(s)}; \bxi_j^{s-d_j(s)}) - \nabla F_j(\bw^{s-\tau_j(s)}) , \nabla f_j(\bw^{s'-\tau_j(s')}; \bxi_j^{s'-d_j(s')}) - \nabla F_j(\bw^{s'-\tau_j(s')}) \right\rangle
			\nonumber
			\\
			&~+ \sum_{i,j: i \neq j} \underbrace{\E \left\langle \nabla f_i(\bw^{s-\tau_i(s)}; \bxi_i^{s-d_i(s)}) - \nabla F_i(\bw^{s-\tau_i(s)}) , \nabla f_j(\bw^{s'-\tau_j(s')}; \bxi_j^{s'-d_j(s')}) - \nabla F_j(\bw^{s'-\tau_j(s')}) \right\rangle}_{Z^{ss'}_{ij}}
			\label{eq:Z}
		\end{align}
		To simplify $Z^{ss'}_{ij}$ for all $i,j \in [n]$ and $i \neq j$, we assume with out loss of generality that $s-d_i(s) \geq s'-d_j(s')$. This, together with the  fact that $d_j(s') \leq \tau_j(s')$ by \eqref{eq:tau-d}, implies that $s-d_i(s) \geq s' - \tau_j(s')$.
		Thus, $\bxi_j^{s-d_i(s)}$ is independent of $\bw^{s'-d_j(s')}$.
		Further using the law of total expectation and Assumption \ref{as:unbiased}, we have
		\begin{align}
			Z^{ss'}_{ij} =&~ \E \left[ \E_{\bxi_i^{s-d_i(s)}} \left\langle \nabla f_i(\bw^{s-\tau_i(s)}; \bxi_j^{s-d_i(s)}) - \nabla F_i(\bw^{s-\tau_i(s)}) , \nabla f_j(\bw^{s'-\tau_i(s')}; \bxi_j^{s'-d_j(s')}) - \nabla F_j(\bw^{s'-\tau_i(s')}) \right\rangle \right]
			\nonumber
			\\
			=&~\E \left[ \left\langle \E_{\bxi_i^{s-d_i(s)}} \left[ \nabla f_i(\bw^{s-\tau_i(s)}; \bxi_i^{s-d_i(s)}) - \nabla F_i(\bw^{s-\tau_i(s)}) \right], 
			\nabla f_j(\bw^{s'-\tau_i(s')}; \bxi_j^{s'-d_j(s')}) - \nabla F_j(\bw^{s'-\tau_i(s')}) \right\rangle \right]
			\nonumber
			\\
			=&~ 0
			\nonumber
		\end{align}
		Substituting this back into \eqref{eq:Z} and using Assumption \ref{as:sigma}, we have
		\begin{align}
			Z^{ss'} =&~ \sum_{j=1}^{n} \E \left\langle \nabla f_j(\bw^{s-\tau_j(s)}; \bxi_j^{t-d_j(s)}) - \nabla F_j(\bw^{s-\tau_j(s)}) , \nabla f_j(\bw^{s'-\tau_i(s')}; \bxi_j^{s'-d_j(s')}) - \nabla F_j(\bw^{s'-\tau_i(s')}) \right\rangle
			\nonumber
			\\
			\leq&~ \frac{1}{2} \sum_{j=1}^{n} \E \| \nabla f_j(\bw^{s-\tau_j(s)}; \bxi_j^{t-d_j(s)}) - \nabla F_j(\bw^{s-\tau_j(s)}) \|_2^2
			\nonumber
			\\
			&~ + \frac{1}{2} \sum_{j=1}^{n}  \E \| \nabla f_j(\bw^{s'-\tau_i(s')}; \bxi_j^{s'-d_j(s')}) - \nabla F_j(\bw^{s'-\tau_i(s')}) \|_2^2 
			\nonumber
			\\
			\leq&~ n \sigma^2.
			\nonumber
		\end{align}
		Plugging this back into \eqref{eq:Y1} and using Lemma \ref{lem:variance} yield
		\begin{align}
			\Phi_1&\leq \sum_{s = 1+t-\tau_i(t)}^{t} n \sigma^2 + \sum_{\substack{s,s': s \neq s', \\ 1+t-\tau_i(t) \leq s,s' \leq t}} n \sigma^2 
			\nonumber
			\\
			&= \tau_i(t) n \sigma^2 + (\tau_i(t)^2-\tau_i(t)) n \sigma^2 
			\nonumber
			\\
			&= \tau_i(t)^2 n \sigma^2 
			\nonumber
			\\
			&\leq n \tau_{\max}^2 \sigma^2.
			\label{eq:Phi1}
		\end{align}
		
		\underline{Upper bounding $\Phi_2$:} Following the fact that $\|\sum_{i=1}^{m} \bx_i \|_2^2 \leq m \sum_{i=1}^{m} \| \bx_i \|_2^2$ for vectors $\bx_1,\dots,\bx_m$, we have 
		\begin{align}
			\Phi_2 &= \E \left\| \sum_{s = 1+t-\tau_i(t)}^{t} \sum_{j=1}^{n} \nabla F_j(\bw^{s-\tau_j(s)}) \right\|_2^2
			\nonumber
			\\
			&\leq \tau_i(t) \sum_{s = 1+t-\tau_i(t)}^{t} \E \left\| \sum_{j=1}^{n} \nabla F_j(\bw^{s-\tau_j(s)}) \right\|_2^2.
			\nonumber
			\\
			&\leq \tau_{\max} \sum_{s=1+[t-\tau_{\max}]_+}^{t} \E \left\| \sum_{j=1}^{n} \nabla F_j(\bw^{s-\tau_j(s)}) \right\|_2^2.
			\label{eq:Y2}
		\end{align}
		Substituting \eqref{eq:Phi1} and \eqref{eq:Y2} back into \eqref{eq:diff-tmp} gives
		\begin{align}
			\E \| \bw^t - \bw^{t-\tau_i(t)} \|_2^2
			\leq \frac{2 \sigma^2}{n} \tau_{\max}^2 \eta^2 + 2 \tau_{\max} \eta^2 \sum_{s=1+[t-\tau_{\max}]_+}^{t} \E \left\| \frac{1}{n} \sum_{j=1}^{n} \nabla F_j(\bw^{s-\tau_j(s)}) \right\|_2^2,
			\nonumber
		\end{align}
		as desired.
	\end{proof}
	
	\begin{lemma}\label{lem:bb}
		Suppose that Assumptions \ref{as:function}--\ref{as:sigma} hold. Then, it holds for all $i \in [n]$ and $t \geq 1$ that
		\begin{align}
			\E \| \bm{g}^t \|_2^2
			\leq& \left( 2 + 8 L^2 \tau_{\max}^2 \eta^2 \right)\frac{\sigma^2}{n}
			+ 4 \E \| \nabla F (\bw^{t-1}) \|_2^2
			\nonumber
			\\
			&+ 8 L^2 \tau_{\max} \eta^2 \sum_{s=1+[t-\tau_{\max}]_+}^{t} \E \left\| \frac{1}{n} \sum_{j=1}^{n} \nabla F_j(\bw^{s-\tau_j(s)}) \right\|_2^2.
			\nonumber
		\end{align}
	\end{lemma}
	\begin{proof}
		Following the fact that $\| \bx + \by \|_2^2 \leq 2 \| \bx \|_2^2 + 2 \| \by \|_2^2$ for vectors $\bx$ and $\by$, we have
		\begin{align}
			&~\E \|\bm{g}^t\|_2^2 
			\nonumber
			\\
			=&~ \E \left\| \frac{1}{n} \sum_{i=1}^{n} \nabla f_i ( \bw^{t-\tau_i(t)}; \bxi_i^{t-d_i(t)} ) \right\|_2^2
			\nonumber
			\\
			=&~ \E \left\| \frac{1}{n} \sum_{i=1}^{n} \left(\nabla f_i( \bw^{t-\tau_i(t)}; \bxi_i^{t-d_i(t)} ) - \nabla F_i( \bw^{t-\tau_i(t)} ) \right) 
			+ \frac{1}{n} \sum_{i=1}^{n} \nabla F_i( \bw^{t-\tau_i(t)} ) \right\|_2^2
			\nonumber
			\\
			\leq&~ 2 \E \left\| \frac{1}{n} \sum_{i=1}^{n} \left(\nabla f_i( \bw^{t-\tau_i(t)}; \bxi_i^{t-d_i(t)} ) - \nabla F_i( \bw^{t-\tau_i(t)} ) \right)   \right\|_2^2
			+ 2 \E \left\| \frac{1}{n} \sum_{i=1}^{n} \nabla F_i( \bw^{t-\tau_i(t)} ) \right\|_2^2
			\nonumber
			\\
			\leq&~ \frac{2\sigma^2}{n} 
			+ 2 \underbrace{\E \left\| \frac{1}{n} \sum_{i=1}^{n} \nabla F_i( \bw^{t-\tau_i(t)} ) \right\|_2^2}_{\Psi}.
			\label{eq:XYZ}
		\end{align}
		where the last inequality holds due to Lemma \ref{lem:variance}.
		It suffices to upper bound $\Psi$. We observe that
		\begin{align}
			\Psi
			&= \E \left\| \frac{1}{n} \sum_{i=1}^{n} \nabla F_i( \bw^{t-\tau_i(t)} ) \right\|_2^2
			\nonumber
			\\
			&= \E \left\| \frac{1}{n} \sum_{i=1}^{n} \left( \nabla F_i( \bw^{t-\tau_i(t)} ) - \nabla F_i( \bw^t ) \right) 
			+ \frac{1}{n} \sum_{i=1}^{n} \nabla F_i(\bw^{t-1}) \right\|_2^2
			\nonumber
			\\
			&\leq 2 \E \left\| \frac{1}{n} \sum_{i=1}^{n} \left( \nabla F_i( \bw^{t-\tau_i(t)} ) - \nabla F_i( \bw^t ) \right) \right\|_2^2
			+ 2 \E \left\| \frac{1}{n} \sum_{i=1}^{n} \nabla F_i(\bw^{t-1}) \right\|_2^2
			\nonumber
			\\
			&\leq \frac{2}{n} \sum_{i=1}^{n} \E \| \nabla F_i( \bw^{t-\tau_i(t)} ) - \nabla F_i( \bw^t ) \|_2^2
			+ 2 \E \| \nabla F (\bw^{t-1}) \|_2^2
			\nonumber
			\\
			&\leq \frac{2 L^2}{n} \sum_{i=1}^{n} \E \| \bw^{t-\tau_i(t)} - \bw^t \|_2^2
			+ 2 \E \| \nabla F (\bw^{t-1}) \|_2^2,
			\label{eq:X2}
		\end{align}
		where the second inequality uses the fact that $\|\sum_{i=1}^{m} \bx_i \|_2^2 \leq m \sum_{i=1}^{m} \| \bx_i \|_2^2$ for vectors $\bx_1,\dots,\bx_m$.
		Substituting Lemma \ref{lem:diff} into \eqref{eq:X2} gives
		\begin{align}
			\Psi &\leq 2 L^2 \left( \frac{2 \sigma^2}{n} \tau_{\max}^2 \eta^2 + 2 \tau_{\max} \eta^2 \sum_{s=1+[t-\tau_{\max}]_+}^{t} \E \left\| \frac{1}{n} \sum_{j=1}^{n} \nabla F_j(\bw^{s-\tau_j(s)}) \right\|_2^2 \right)
			+ 2 \E \| \nabla F (\bw^{t-1}) \|_2^2
			\nonumber
			\\
			&= \frac{4 \sigma^2}{n} L^2 \tau_{\max}^2 \eta^2 
			+ 4 L^2 \tau_{\max} \eta^2 \sum_{s=1+[t-\tau_{\max}]_+}^{t} \E \left\| \frac{1}{n} \sum_{j=1}^{n} \nabla F_j(\bw^{s-\tau_j(s)}) \right\|_2^2
			+ 2 \E \| \nabla F (\bw^{t-1}) \|_2^2.
			\label{eq:Psi}
		\end{align}
		Plugging \eqref{eq:Psi} back into \eqref{eq:XYZ} gives
		\begin{align}
			\E \| \bm{g}^t \|_2^2
			\leq& \left( 2 + 8 L^2 \tau_{\max}^2 \eta^2 \right)\frac{\sigma^2}{n}
			+ 4 \E \| \nabla F (\bw^{t-1}) \|_2^2
			\nonumber
			\\
			&+ 8 L^2 \tau_{\max} \eta^2 \sum_{s=1+[t-\tau_{\max}]_+}^{t} \E \left\| \frac{1}{n} \sum_{j=1}^{n} \nabla F_j(\bw^{s-\tau_j(s)}) \right\|_2^2,
			\nonumber
		\end{align}
		as desired.
	\end{proof}
	
	\subsection{Proof of Proposition \ref{prop:cc}}
	\label{appen:prop}
	\begin{proof}
		We first decompose the inner product into two terms:
		\begin{align}
			\E \langle \nabla F(\bw^{t-1}), \bg^t \rangle
			= \underbrace{\E \left\langle \nabla F(\bw^{[t-\tau_{\max}]_+}), \bg^t \right\rangle}_{A}
			+ \underbrace{\E \left\langle \nabla F (\bw^{t-1}) - \nabla F(\bw^{[t-\tau_{\max}]_+}), \bg^t \right\rangle}_{B}.
			\label{eq:A+B}
		\end{align}
		Subsequently, we upper bound $A_1$ and $A_2$, respectively.
		
		\underline{Lower bounding $A$:} 
		Since $d_i(t) \leq \tau_i(t) \leq \tau_{\max}$ for all $i \in [n]$, then $t - d_i(t) \geq t - \tau_{\max}$ for all $i \in [n]$, which implies that $\bxi_1^{t-d_1(t)},\dots,\bxi_n^{t-d_n(t)}$ are independent of $\bw^{[t-\tau_{\max}]_+}$. Then, we have
		\begin{align}
			A =&~ \E \left[ \left\langle \nabla F(\bw^{[t-\tau_{\max}]_+}), \bg^t \right\rangle \right]
			\nonumber
			\\
			\stackrel{(a)}{=}&~ \E \left[ \E_{\bxi_1^{t-d_1(t)},\dots,\bxi_n^{t-d_n(t)}} \left[ \left\langle \nabla F(\bw^{[t-\tau_{\max}]_+}), \bg^t \right\rangle \right] \right]
			\nonumber
			\\
			=&~\E \left\langle \nabla F(\bw^{[t-\tau_{\max}]_+}), \E_{\bxi_i^{t-d_i(t)}} \left[ \frac{1}{n} \sum_{i=1}^{n} \nabla f_i (\bw^{t-\tau_i(t)}; \bxi_i^{t-d_i(t)}) \right] \right\rangle 
			\nonumber
			\\
			\stackrel{(b)}{=}&~ \E \left\langle \nabla F(\bw^{[t-\tau_{\max}]_+}), \frac{1}{n} \sum_{i=1}^{n} \nabla F_i (\bw^{t-\tau_i(t)}) \right\rangle 
			\nonumber
			\\
			=&~ \underbrace{\E \left\langle \nabla F(\bw^{[t-\tau_{\max}]_+}) - \nabla F (\bw^{t-1}), \frac{1}{n} \sum_{i=1}^{n} \nabla F_i (\bw^{t-\tau_i(t)}) \right\rangle}_{A_1} 
			\nonumber
			\\
			&+ \underbrace{\E \left\langle \nabla F (\bw^{t-1}), \frac{1}{n} \sum_{i=1}^{n} \nabla F_i (\bw^{t-\tau_i(t)}) \right\rangle}_{A_2},
			\label{eq:A1+A2}
		\end{align}
		where $(a)$ use the law of total expectation, and $(b)$ holds due to Assumption \ref{as:unbiased}.
		Then, we lower bound $A_1$ as follows:
		\begin{align}
			A_1 =&~ \E \left\langle \nabla F(\bw^{[t-\tau_{\max}]_+}) - \nabla F (\bw^{t-1}), \frac{1}{n} \sum_{i=1}^{n} \left( \nabla F_i (\bw^{t-\tau_i(t)}) - \nabla F_i (\bw^{t-1}) \right) \right\rangle
			\nonumber
			\\
			&+  \E \left\langle \nabla F(\bw^{[t-\tau_{\max}]_+}) - \nabla F (\bw^{t-1}), \nabla F (\bw^{t-1}) \right\rangle
			\nonumber
			\\
			\geq& - \frac{1}{2} \E \| \nabla F(\bw^{[t-\tau_{\max}]_+}) - \nabla F (\bw^{t-1}) \|_2^2
			- \frac{1}{2} \E \left\|  \frac{1}{n} \sum_{i=1}^{n} \left( \nabla F_i (\bw^{t-\tau_i(t)}) - \nabla F_i (\bw^{t-1}) \right) \right\|_2^2
			\nonumber
			\\
			& - \E \| \nabla F(\bw^{[t-\tau_{\max}]_+}) - \nabla F (\bw^{t-1}) \|_2^2
			- \frac{1}{4} \E \| \nabla F (\bw^{t-1}) \|_2^2
			\nonumber
			\\
			=& - \frac{1}{4} \E \| \nabla F (\bw^{t-1}) \|_2^2
			- \frac{3}{2} \E \| \nabla F(\bw^{[t-\tau_{\max}]_+}) - \nabla F (\bw^{t-1}) \|_2^2
			\nonumber
			\\
			&- \frac{1}{2} \E \left\|  \frac{1}{n} \sum_{i=1}^{n} \left( \nabla F_i (\bw^{t-\tau_i(t)}) - \nabla F_i (\bw^{t-1}) \right) \right\|_2^2,
			\label{eq:A1}
		\end{align}
		where the inequality uses the fact that $\langle \bx, \by \rangle \geq - \frac{1}{2} \|\bx\|_2^2 - \frac{1}{2} \|\by\|_2^2$ and $\langle \bx, \by \rangle \geq - \|\bx\|_2^2 - \frac{1}{4} \|\by\|_2^2$ for vectors $\bx$ and $\by$.
		Using the identity $\langle \bx, \by \rangle = \frac{1}{2} \| \bx \|_2^2 + \frac{1}{2} \| \by \|_2^2 - \frac{1}{2} \| \bx - \by \|_2^2$ for vectors $\bx$ and $\by$, we can express $A_2$ as
		\begin{align}
			A_2 =&~ \frac{1}{2} \E \| \nabla F (\bw^{t-1}) \|_2^2
			+ \frac{1}{2} \E \left\| \frac{1}{n} \sum_{i=1}^{n} \nabla F_i (\bw^{t-\tau_i(t)}) \right\|_2^2
			\nonumber
			\\
			&- \frac{1}{2} \E \left\| \nabla F (\bw^{t-1}) - \frac{1}{n} \sum_{i=1}^{n} \nabla F_i (\bw^{t-\tau_i(t)}) \right\|_2^2.
			\label{eq:A2}
		\end{align}
		Putting \eqref{eq:A1} and \eqref{eq:A2} back into \eqref{eq:A1+A2} gives
		\begin{align}
			A \geq&~ \frac{1}{4} \E \| \nabla F (\bw^{t-1}) \|_2^2
			- \frac{3}{2} \E \| \nabla F (\bw^{t-1}) - \nabla F(\bw^{[t-\tau_{\max}]_+}) \|_2^2
			\nonumber
			\\
			& - \E \left\| \frac{1}{n} \sum_{i=1}^{n} \left( \nabla F_i(\bw^{t-1}) - \nabla F_i (\bw^{t-\tau_i(t)}) \right) \right\|_2^2
			+ \frac{1}{2} \E \left\| \frac{1}{n} \sum_{i=1}^{n} \nabla F_i (\bw^{t-\tau_i(t)}) \right\|_2^2
			\nonumber
			\\
			\stackrel{(a)}{=}&~ \frac{1}{4} \E \| \nabla F (\bw^{t-1}) \|_2^2
			- \frac{3L^2}{2} \E \| \bw^t - \bw^{[t-\tau_{\max}]_+} \|_2^2
			\nonumber
			\\
			& - \frac{L^2}{n} \sum_{i=1}^{n} \E \| \bw^t - \bw^{t-\tau_i(t)} \|_2^2
			+ \frac{1}{2} \E \left\| \frac{1}{n} \sum_{i=1}^{n} \nabla F_i (\bw^{t-\tau_i(t)}) \right\|_2^2
			\nonumber
			\\
			\stackrel{(b)}{\geq}&~ \frac{1}{4} \E \| \nabla F (\bw^{t-1}) \|_2^2
			+ \frac{1}{2} \E \left\| \frac{1}{n} \sum_{i=1}^{n} \nabla F_i (\bw^{t-\tau_i(t)}) \right\|_2^2
			\nonumber
			\\
			& - \frac{5L^2}{2} \left( 2 \tau_{\max}^2 \eta^2 \frac{\sigma^2}{n} 
			+ 2 \tau_{\max} \eta^2 \sum_{s=1+[t-\tau_{\max}]_+}^{t} \E \left\| \frac{1}{n} \sum_{j=1}^{n} \nabla F_j (\bw^{s-\tau_j(s)}) \right\|_2^2 \right)
			\nonumber
			\\
			=&~ \frac{1}{4} \E \| \nabla F (\bw^{t-1}) \|_2^2
			- 5 L^2 \tau_{\max}^2 \eta^2 \frac{\sigma^2}{n}
			+ \frac{1}{2} \E \left\| \frac{1}{n} \sum_{i=1}^{n} \nabla F_i (\bw^{t-\tau_i(t)}) \right\|_2^2
			\nonumber
			\\
			&- 5 L^2 \tau_{\max} \eta^2 \sum_{s=1+[t-\tau_{\max}]_+}^{t} \E \left\| \frac{1}{n} \sum_{j=1}^{n} \nabla F_j (\bw^{s-\tau_j(s)}) \right\|_2^2,
			\label{eq:A-lb}
		\end{align}
		where $(a)$ uses Assumption \ref{as:function} and $(b)$ uses Lemma \ref{lem:diff}.
		
		\underline{Lower bounding $B$:}
		We observe that
		\begin{align}
			B &= \E \left\langle \nabla F (\bw^{t-1}) - \nabla F(\bw^{[t-\tau_{\max}]_+}), \bg^t \right\rangle
			\nonumber
			\\
			&\stackrel{(a)}{\geq} - \E \left[ \| \nabla F (\bw^{t-1}) - \nabla F(\bw^{[t-\tau_{\max}]_+}) \|_2 \| \bg^t \|_2 \right]
			\nonumber
			\\
			&\stackrel{(b)}{\geq} - L \E \left[ \| \bw^t - \bw^{[t-\tau_{\max}]_+} \|_2 \| \bg^t \|_2 \right]
			\nonumber
			\\
			&\stackrel{(c)}{\geq} - L \E \left[ \left\| \sum_{s=1+[t-\tau_{\max}]_+}^{t} \eta \bg^s \right\|_2 \| \bg^t \|_2 \right]
			\nonumber
			\\
			&\stackrel{(d)}{\geq} - L \E \left[ \sum_{s=1+[t-\tau_{\max}]_+}^{t} \eta \| \bg^s \|_2 \| \bg^t \|_2 \right]
			\nonumber
			\\
			&\stackrel{(e)}{\geq} - L \eta \sum_{s=1+[t-\tau_{\max}]_+}^{t} \frac{1}{2} \left( \E \| \bg^s \|_2^2 + \E \| \bg^t \|_2^2 \right)
			\nonumber
			\\
			&= - \frac{L \eta}{2} \sum_{s=1+[t-\tau_{\max}]_+}^{t}  \E \| \bg^s \|_2^2 
			- \frac{L\eta}{2} \tau_{\max} \E \| \bg^t \|_2^2,
			\label{eq:B}
		\end{align}
		where $(a)$ follows from the Cauchy-Schwartz inequality, $(b)$ follows from Assumption \ref{as:function}, $(c)$ uses the telescoping sum 
		$\bw^t - \bw^{[t-\tau_{\max}]_+}
		= \sum_{s = 1+[t-\tau_{\max}]_+}^{t} \left( \bw^s - \bw^{s-1} \right)
		= \sum_{s = 1+[t-\tau_{\max}]_+}^{t} \eta \bg^s$,
		$(d)$ uses the triangle inequality, and $(e)$ is due to the Young's inequality.
		Combining \eqref{eq:B} with Lemma \ref{lem:bb}, we have
		\begin{align}
			B \geq&- \frac{L \eta}{2} \sum_{s=1+[t-\tau_{\max}]_+}^{t} \left( \left( 2 + 8 L^2 \tau_{\max}^2 \eta^2 \right)\frac{\sigma^2}{n}
			+ 4 \E \| \nabla F(\bw^{s-1}) \|_2^2 \right.
			\nonumber
			\\
			&~\left. \qquad\qquad\qquad\qquad\qquad + 8 L^2 \tau_{\max} \eta^2 \sum_{s'=1+[s-\tau_{\max}]_+}^{s} \E \left\| \frac{1}{n} \sum_{j=1}^{n} \nabla F_j(\bw^{s'-\tau_i(s')}) \right\|_2^2
			\right)
			\nonumber
			\\
			& - \frac{L \eta}{2} \tau_{\max} \left( 
			\left( 2 + 8 L^2 \tau_{\max}^2 \eta^2 \right)\frac{\sigma^2}{n}
			+ 4 \E \| \nabla F (\bw^{t-1}) \|_2^2 \right.
			\nonumber
			\\
			&\left. \qquad\qquad\qquad\quad + 8 L^2 \tau_{\max} \eta^2 \sum_{s=1+[t-\tau_{\max}]_+}^{t} \E \left\| \frac{1}{n} \sum_{j=1}^{n} \nabla F_j(\bw^{s-\tau_j(s)}) \right\|_2^2
			\right)
			\nonumber
			\\
			=& - (2L\tau_{\max}\eta + 8 L^3 \tau_{\max}^3 \eta^3) \frac{\sigma^2}{n}
			- 2 L \eta \sum_{s=1+[t-\tau_{\max}]_+}^{t} \E \| \nabla F(\bw^{s-1}) \|_2^2
			- 2 L \tau_{\max} \eta \E \| \nabla F (\bw^{t-1}) \|_2^2 
			\nonumber
			\\
			&- 4 L^3 \tau_{\max} \eta^3 \sum_{s=1+[t-\tau_{\max}]_+}^{t} \sum_{s'=1+[s-\tau_{\max}]_+}^{s} \E \left\| \frac{1}{n} \sum_{j=1}^{n} \nabla F_j(\bw^{s'-\tau_i(s')}) \right\|_2^2
			\nonumber
			\\
			& - 4 L^3 \tau_{\max}^2 \eta^3 \sum_{s=1+[t-\tau_{\max}]_+}^{t} \E \left\| \frac{1}{n} \sum_{j=1}^{n} \nabla F_j(\bw^{s-\tau_j(s)}) \right\|_2^2 
			\nonumber
			\\
			\geq& - (2L\tau_{\max}\eta + 8 L^3 \tau_{\max}^3 \eta^3) \frac{\sigma^2}{n}  
			- 2 L \eta \sum_{s=1+[t-\tau_{\max}]_+}^{t} \E \| \nabla F(\bw^{s-1}) \|_2^2
			- 2 L \tau_{\max} \eta \E \| \nabla F (\bw^{t-1}) \|_2^2 
			\nonumber
			\\
			& - 8 L^3 \tau_{\max}^2 \eta^3 \sum_{s=1+[t-2\tau_{\max}]_+}^{t} \E \left\| \frac{1}{n} \sum_{j=1}^{n} \nabla F_j(\bw^{s-\tau_j(s)}) \right\|_2^2,
			\label{eq:B-lb}
		\end{align}
		where the last inequality uses the fact that
		$\sum_{s=1+[t-K]_+}^{t} \sum_{s'=1+[s-K]_+}^{s} a_{s'} \leq K \sum_{s=1+[t-2K]_+}^{t} a_s$
		for $a_1,\dots,a_t \geq 0$ and $K\geq1$.
		
		Substituting \eqref{eq:A-lb} and \eqref{eq:B-lb} into \eqref{eq:A+B} and simplifying it, we have
		\begin{align}
			&~\E \langle \nabla F (\bw^{t-1}), \bg^t \rangle
			\nonumber
			\\
			\geq& \left(\frac{1}{4} - 2 L \tau_{\max} \eta\right) \E \| \nabla F (\bw^{t-1}) \|_2^2 
			- 2 L \eta \sum_{s=1+[t-\tau_{\max}]_+}^{t} \E \| \nabla F(\bw^{s-1}) \|_2^2 
			\nonumber
			\\
			&- (2 L \tau_{\max} \eta + 5 L^2 \tau_{\max}^2 \eta^2 + 8 L^3 \tau_{\max}^3 \eta^3) \frac{\sigma^2}{n}
			+ \frac{1}{2} \E \left\| \frac{1}{n} \sum_{i=1}^{n} \nabla F_i (\bw^{t-\tau_i(t)}) \right\|_2^2
			\nonumber
			\\
			& - (5 L^2 \tau_{\max} \eta^2 + 8 L^3 \tau^2_{\max} \eta^3) \sum_{s=1+[t-2\tau_{\max}]_+}^{t} \E \left\| \frac{1}{n} \sum_{j=1}^{n} \nabla F_j (\bw^{s-\tau_j(s)}) \right\|_2^2
			\nonumber
			\\
			\geq&~ \frac{1}{8} \E \| \nabla F (\bw^{t-1}) \|_2^2 
			- 2 L \eta \sum_{s=1+[t-\tau_{\max}]_+}^{t} \E \| \nabla F(\bw^{s-1}) \|_2^2 
			- 3 L \tau_{\max} \eta \frac{\sigma^2}{n}
			\nonumber
			\\
			&~ 
			+ \frac{1}{2} \E \left\| \frac{1}{n} \sum_{i=1}^{n} \nabla F_i (\bw^{t-\tau_i(t)}) \right\|_2^2
			- 6 L^2 \tau_{\max} \eta^2 \sum_{s=1+[t-2\tau_{\max}]_+}^{t} \E \left\| \frac{1}{n} \sum_{j=1}^{n} \nabla F_j (\bw^{s-\tau_j(s)}) \right\|_2^2,
			\nonumber
		\end{align}
		where the last inequality holds because the stepsize condition $\eta \leq 1/(16 L \tau_{\max})$ implies the following:
		\begin{align}
			\frac{1}{4} - 2 L \tau_{\max} \eta &\geq \frac{1}{8},
			\nonumber
			\\
			2 L \tau_{\max} \eta + 5 L^2 \tau_{\max}^2 \eta^2 + 8 L^3 \tau_{\max}^3 \eta^3 &\leq 2 L \tau_{\max} \eta + 6 L^2 \tau_{\max}^2 \eta^2 \leq 3 L \tau_{\max} \eta,
			\nonumber
			\\
			5 L^2 \tau_{\max} \eta^2 + 8 L^3 \tau_{\max}^2 \eta^3
			&\leq 6 L^2 \tau_{\max} \eta^2.
			\nonumber
		\end{align}
		This completes the proof.
	\end{proof}
	
	\subsection{Proof of Theorem \ref{thm}}
	\label{appen:thm}
	\begin{proof}
		Since $F$ is $L$-smooth, it follows from the descent lemma that
		\begin{align}
			\E[F(\bw^t)]-\E[F(\bw^{t-1})] 
			\leq&~ \E [ \langle \nabla F(\bw^{t-1}), \bw^t - \bw^{t-1} \rangle ] + \frac{L}{2} \E \| \bw^t - \bw^{t-1} \|_2^2
			\nonumber
			\\
			=& - \eta \E \langle \nabla F(\bw^{t-1}), \bg^t \rangle + \frac{L\eta^2}{2} \E \| \bg^t \|_2^2.
			\nonumber
		\end{align}
		Applying Lemma \ref{lem:bb} and Proposition \ref{prop:cc}, we obtain
		\begin{align}
			&~\E[F(\bw^t)]-\E[F(\bw^{t-1})] 
			\nonumber
			\\
			\leq& - \frac{1}{8}\eta \E \| \nabla F (\bw^{t-1}) \|_2^2 
			+ 2 L \eta^2 \sum_{s=1+[t-\tau_{\max}]_+}^{t} \E \| \nabla F(\bw^{s-1}) \|_2^2 
			+ 3 L \tau_{\max} \eta^2 \frac{\sigma^2}{n}
			\nonumber
			\\
			& - \frac{1}{2} \eta \E \left\| \frac{1}{n} \sum_{i=1}^{n} \nabla F_i (\bw^{t-\tau_i(t)}) \right\|_2^2
			+ 6 L^2 \tau_{\max} \eta^3 \sum_{s=1+[t-2\tau_{\max}]_+}^{t} \E \left\| \frac{1}{n} \sum_{j=1}^{n} \nabla F_j (\bw^{s-\tau_j(s)}) \right\|_2^2
			\nonumber
			\\
			&+ \left( L \eta^2 + 4 L^3 \tau_{\max}^2 \eta^4 \right)\frac{\sigma^2}{n}
			+ 4 L^3 \tau_{\max} \eta^4 \sum_{s=1+[t-\tau_{\max}]_+}^{t} \E \left\| \frac{1}{n} \sum_{j=1}^{n} \nabla F_j(\bw^{s-\tau_j(s)}) \right\|_2^2
			\nonumber
			\\
			&+ 2 L \eta^2 \E \| \nabla F (\bw^{t-1}) \|_2^2
			\nonumber
			\\
			\leq& - \left( \frac{1}{8} \eta - 2 L \eta^2 \right) \E \| \nabla F (\bw^{t-1}) \|_2^2  
			+ 2 L \eta^2 \sum_{s=1+[t-\tau_{\max}]_+}^{t} \E \| \nabla F(\bw^{s-1}) \|_2^2
			\nonumber
			\\
			&+ (L \eta^2 + 3 L \tau_{\max} \eta^2 + 4 L^3 \tau_{\max}^2  \eta^4) \frac{\sigma^2}{n}
			- \frac{1}{2} \eta \E \left\| \frac{1}{n} \sum_{i=1}^{n} \nabla F_i(\bw^{t-\tau_i(t)}) \right\|_2^2
			\nonumber
			\\
			&+ (6 L^2 \tau_{\max} \eta^3 + 4 L^3 \tau_{\max} \eta^4) \sum_{s=1+[t-2\tau_{\max}]_+}^{t} \E \left\| \frac{1}{n} \sum_{j=1}^{n} \nabla F_j(\bw^{s-\tau_j(s)}) \right\|_2^2
			\nonumber
			\\
			\leq&  - \frac{1}{16} \eta \E \| \nabla F (\bw^{t-1}) \|_2^2 
			+ 2 L \eta^2 \sum_{s=1+[t-\tau_{\max}]_+}^{t} \E \| \nabla F(\bw^{s-1}) \|_2^2
			\nonumber
			\\
			& + (4 L \tau_{\max} \eta^2 + 4 L^3 \tau_{\max}^2 \eta^4) \frac{\sigma^2}{n}
			\nonumber
			\\
			& - \frac{1}{2} \eta \E \left\| \frac{1}{n} \sum_{i=1}^{n} \nabla F_i(\bw^{t-\tau_i(t)}) \right\|_2^2
			+ 7 L^2 \tau_{\max} \eta^3 \sum_{s=1+[t-2\tau_{\max}]_+}^{t} \E \left\| \frac{1}{n} \sum_{j=1}^{n} \nabla F_j(\bw^{s-\tau_j(s)}) \right\|_2^2,
			\label{eq:last}
		\end{align}
		where the last inequality holds because $\tau_{\max} \geq 1 \Rightarrow L \eta^2 + 3 L \tau_{\max} \eta^2 \leq 4 L \tau_{\max} \eta$ and by requiring the following stepsize conditions:
		\begin{align}
			\eta \leq \frac{1}{32 L} 
			\Longleftrightarrow&~
			\frac{1}{8} \eta - 2 L \eta^2 \geq \frac{1}{16} \eta,
			\label{eq:step2}
			\\
			\eta \leq \frac{1}{4 L} 
			\Longrightarrow&~
			6 L^2 \tau_{\max} \eta^3 + 4 L^3 \tau_{\max} \eta^4 \leq 7 L^2 \tau_{\max} \eta^3.
			\label{eq:step3}
		\end{align}
		Summing up both sides of inequality \eqref{eq:last} for $t = 1, \dots, T$ yields
		\begin{align}
			& \E [F(\bw^T)] - F(\bw^0)
			\nonumber
			\\
			\leq& - \frac{1}{16} \eta \sum_{t=1}^{T} \E \| \nabla F (\bw^{t-1}) \|_2^2
			+ 2 L \eta^2 \sum_{t=1}^{T} \sum_{s=1+[t-\tau_{\max}]_+}^{t} \E \| \nabla F(\bw^{s-1}) \|_2^2 
			\nonumber
			\\
			& + \sum_{t=1}^{T} (4 L \tau_{\max} \eta^2 + 4 L^3 \tau_{\max}^2 \eta^4) \frac{\sigma^2}{n}
			\nonumber
			\\
			& - \frac{1}{2} \eta \sum_{t=1}^{T} \E \left\| \frac{1}{n} \sum_{j=1}^{n} \nabla F_j(\bw^{t-\tau_j(t)}) \right\|_2^2 
			+ 7 L^2 \tau_{\max} \eta^3 \sum_{t=1}^{T} \sum_{s=1+[t-2\tau_{\max}]_+}^{t} \E \left\| \frac{1}{n} \sum_{j=1}^{n} \nabla F_j(\bw^{s-\tau_j(s)}) \right\|_2^2
			\nonumber
			\\
			\leq& - \frac{1}{16} \eta \sum_{t=1}^{T} \E \| \nabla F (\bw^{t-1}) \|_2^2
			+ 2 L \tau_{\max} \eta^2 \sum_{t=1}^{T} \E \| \nabla F (\bw^{t-1}) \|_2^2 
			\nonumber
			\\
			& + \sum_{t=1}^{T} (4 L \tau_{\max} \eta^2 + 4 L^3 \tau_{\max}^2 \eta^4) \frac{\sigma^2}{n}
			\nonumber
			\\
			& - \frac{1}{2} \eta \sum_{t=1}^{T} \E \left\| \frac{1}{n} \sum_{j=1}^{n} \nabla F_j(\bw^{t-\tau_j(t)}) \right\|_2^2 
			+ 14 L^2 \tau_{\max}^2 \eta^3 \sum_{t=1}^{T} \E \left\| \frac{1}{n} \sum_{j=1}^{n} \nabla F_j(\bw^{t-\tau_j(t)}) \right\|_2^2
			\nonumber
			\\
			=& - \left( \frac{1}{16} \eta - 2 L \tau_{\max} \eta^2 \right) \sum_{t=1}^{T} \E \| \nabla F (\bw^{t-1}) \|_2^2
			+ \sum_{t=1}^{T} (4 L \tau_{\max} \eta^2 + 4 L^3 \tau_{\max}^2 \eta^4) \frac{\sigma^2}{n}
			\nonumber
			\\
			&- \left( \frac{1}{2} \eta - 14 L^2 \tau_{\max}^2 \eta^3 \right) \sum_{t=1}^{T} \E \left\| \frac{1}{n} \sum_{j=1}^{n} \nabla F_j(\bw^{t-\tau_j(t)}) \right\|_2^2
			\nonumber
			\\
			\leq& - \frac{1}{32} \eta \sum_{t=1}^{T} \E \| \nabla F (\bw^{t-1}) \|_2^2
			+ \sum_{t=1}^{T} (4 L \tau_{\max} \eta^2 + 4 L^3 \tau_{\max}^2 \eta^4) \frac{\sigma^2}{n},
			\label{eq:FT-F0}
		\end{align}
		where the second inequality uses the fact that $\sum_{t=1}^{T} \sum_{s=1+[t-K]_+}^{t} a_{s}
		\leq K \sum_{t=1}^{T} a_t$
		for $a_1,\dots,a_T \geq 0$ and $K \geq 1$, and the last inequality holds by requiring the following stepsize conditions:
		\begin{align}
			\eta \leq \frac{1}{64 L \tau_{\max}} 
			&\Longleftrightarrow
			\frac{1}{16} \eta - 2 L \tau_{\max} \eta^2 \geq \frac{1}{32} \eta,
			\label{eq:step4}
			\\
			\eta \leq \frac{1}{\sqrt{28} L \tau_{\max}} 
			&\Longleftrightarrow \frac{1}{2} \eta - 14 L^2 \tau_{\max}^2 \eta^3 \geq 0.
			\label{eq:step5}
		\end{align}
		Rearranging \eqref{eq:FT-F0} and using Assumption \ref{as:lowerbound}, we obtain 
		\begin{align}
			&\frac{1}{T} \sum_{t=1}^{T} \E \| \nabla F (\bw^{t-1}) \|_2^2 
			\leq \frac{32 (F(\bw^0) - F^*)}{T\eta}  + 128 L \tau_{\max} \eta \frac{\sigma^2}{n}  + 128 L^3 \tau_{\max}^2 \eta^3 \frac{\sigma^2}{n}.
			\label{eq:rearranged}
		\end{align}
		It suffices to choose $\eta$ s.t. the right hind side of \eqref{eq:rearranged} can be minimized.
		To proceed, using the inequality $a + b \geq 2 \sqrt{ab}$ for $a,b \geq 0$, we note that
		\begin{align}
			\frac{32 (F(\bw^0) - F^*)}{T\eta}  + 128 L \tau_{\max} \eta \frac{\sigma^2}{n} \geq 128 \sqrt{\frac{L \sigma^2 \tau_{\max} (F(\bw^0) - F^*)}{nT}},
			\nonumber
		\end{align}
		where the equality holds if and only if
		\begin{align}
			\frac{32 (F(\bw^0) - F^*)}{T\eta} = 128 L \tau_{\max} \eta \frac{\sigma^2}{n}
			\Longleftrightarrow
			\eta = \frac{1}{2}\sqrt{\frac{n(F(\bw^0) - F^*)}{L \sigma^2 \tau_{\max} T}}.
		\end{align}
		Note that the stepsize conditions \eqref{eq:step2}, \eqref{eq:step3}, \eqref{eq:step4}, and \eqref{eq:step5} are implied by $\eta \leq 1/(64 L \tau_{\max})$.
		We take $\eta = \frac{1}{2} \sqrt{\frac{n(F(\bw^0) - F^*)}{L \sigma^2 \tau_{\max} T}}$, then the stepsize conditions can be satisfied when
		\begin{align}
			\frac{1}{2}\sqrt{\frac{n(F(\bw^0) - F^*)}{L \sigma^2 \tau_{\max} T}} \leq \frac{1}{64 L \tau_{\max}}
			\Longleftrightarrow
			T \geq \frac{1024 L (F(\bw^0) - F^*) n \tau_{\max}}{\sigma^2}.
			\label{eq:T-lb}
		\end{align}
		Following from \eqref{eq:rearranged}, when the number of iterations satisfies \eqref{eq:T-lb}, we have
		\begin{align}
			&\frac{1}{T} \sum_{t=1}^{T} \E \| \nabla F (\bw^{t-1}) \|_2^2 
			\nonumber
			\\
			\leq&~ 128 \sqrt{\frac{L \sigma^2 \tau_{\max} (F(\bw^0) - F^*)}{nT}} 
			+ 128 L^3  \tau_{\max}^2 \left(\sqrt{\frac{n(F(\bw^0) - F^*)}{4 L \sigma^2 \tau_{\max} T}}\right)^3 \frac{\sigma^2}{n} 
			\nonumber
			\\
			=&~ 128 \sqrt{\frac{L \sigma^2  \tau_{\max} (F(\bw^0) - F^*)}{nT}} +  \frac{128((F(\bw^0) - F^*) L)^{3/2} \sqrt{n\tau_{\max}}}{\sigma   T^{3/2}},
		\end{align}
		which completes the proof.
	\end{proof}
	
	\section{Additional Experimental Details and Numerical Results}\label{sec:appendix-numerical}
	
	\textbf{Data Partitioning.} Following the approaches adopted in many works \citep{yurochkin2019bayesian,hsu2019measuring,li2022federated}, we use Dirichlet distribution to split the CIFAR-10 dataset into $n$ subsets. 
	The training set in CIFAR-10 consists of 50,000 images with 10 different classes.
	For each class $k \in [10]$, we generate a generate a vector $\bm{p}_k \in \mathbb{R}^n$ from the $n$-dimensional Dirichlet distribution with concentration parameter $\alpha$, whose probability density is given by 
	\[
	\text{Dir}_n(\bm{p}_k; \alpha) \coloneqq \frac{1}{B(\alpha)} \prod_{i=1}^n p_{k,i}^{\alpha - 1}.
	\] 
	Here, $B(\alpha) \coloneqq \frac{\prod_{i=1}^n \Gamma(\alpha)}{\Gamma(n \alpha)}$ is the Beta function, $\Gamma(\cdot)$ is the Gamma function, and $\bm{p}_k$ satisfies $p_{k,i} \in [0,1]$ and $\sum_{i=1}^{n} p_{k,i} = 1$. 
	After generating $\bm{p}_1, \dots, \bm{p}_{10}$, each instance of class $k$ is assigned to client $i$ with probability $p_{k,i}$.
	
	\textbf{Numerical Results for $n=30$ Workers.} 
	We conduct experiments with a configuration of $n=30$ workers. Increasing the number of workers $n$ in the Dirichlet distribution with a given concentration parameter $\alpha$ tends to result in more balanced data partitioning. For our experiments, we select $\alpha = 0.05$ and $0.1$ to observe the effects with $n=30$.
	Each experiment is independently conducted three times using different random seeds. We report the mean and standard deviation of the numerical performance for this configuration, as illustrated in Figure \ref{fig:loss2}.
	We observe that DuDe-ASGD displays similar performance patterns to other algorithms, as previously shown in Figure \ref{fig:loss}, across different levels of data heterogeneity.
	
	\begin{figure}[ht]
		\centering
		\begin{subfigure}[b]{\textwidth}
			\centering
			\includegraphics[width=\textwidth]{figures/legend.pdf}
			\vspace{-7mm}
		\end{subfigure}
		\begin{subfigure}[t]{\textwidth}
			\centering
			\includegraphics[width=\textwidth]{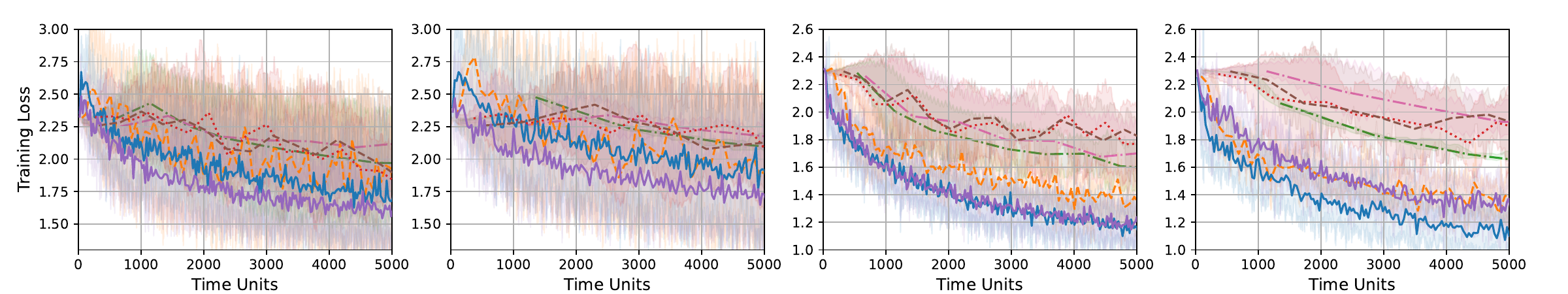}
		\end{subfigure}
		\begin{subfigure}[t]{\textwidth},
			\centering
			\includegraphics[width=\textwidth]{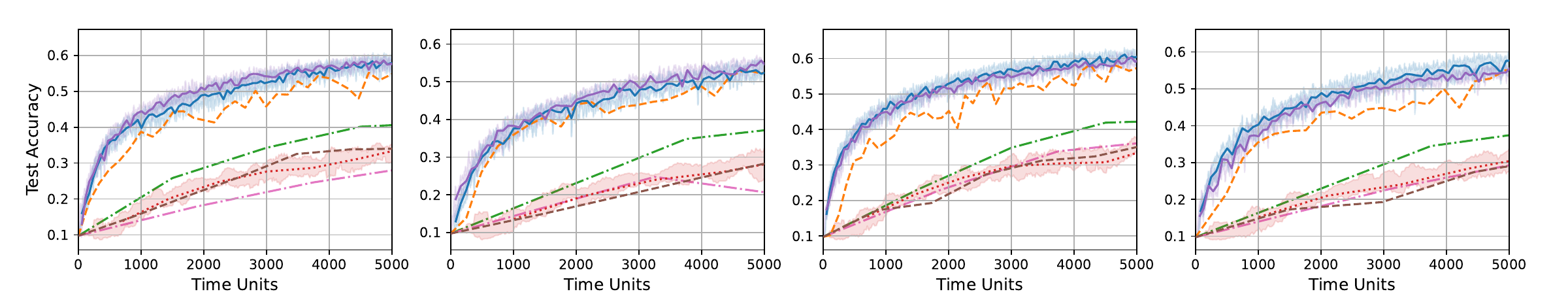}
		\end{subfigure}
		\caption{Convergence curves displaying training losses and test accuracies over time with $n=30$ workers. (1st column: $\alpha=0.05, \texttt{std}=1$; 2nd column: $\alpha=0.05, \texttt{std}=5$; 3rd column: $\alpha=0.1, \texttt{std}=1$; 4th column: $\alpha=0.1, \texttt{std}=5$)
		}
		\label{fig:loss2}
	\end{figure}
	
\end{document}